\documentclass[acmlarge]{acmart}
\makeatletter
\newcommand{\confshort}{\acmConference@shortname}
\newcommand{\conffull}{\acmConference@name}
\newcommand{\confdate}{\acmConference@date}
\newcommand{\confloc}{\acmConference@venue}
\AtBeginDocument{
  \fancypagestyle{firstpagestyle}{
    \fancyhead{}%
    \fancyfoot[C]{}%
  }
  \fancyhf{}
  \fancyhead[LO]{\@headfootfont\shorttitle}%
  \fancyhead[RE]{\@headfootfont\@shortauthors}%
  \fancyhead[LE]{\@headfootfont\footnotesize \confshort, \confdate, \confloc}%
  \fancyhead[RO]{\@headfootfont\footnotesize \confshort, \confdate, \confloc}%
  \fancyfoot[C]{}%
}
\makeatother
\AtBeginDocument{%
  }
\acmBooktitle{\conffull\@ (\confshort), \confdate, \confloc}

\copyrightyear{2026}
\acmYear{2026}
\setcopyright{cc}
\setcctype{by}
\acmConference[FAccT '26]{The 2026 ACM Conference on Fairness, Accountability, and Transparency}{June 25--28, 2026}{Montreal, QC, Canada}
\acmBooktitle{The 2026 ACM Conference on Fairness, Accountability, and Transparency (FAccT '26), June 25--28, 2026, Montreal, QC, Canada}
\acmDOI{10.1145/3805689.3812359}
\acmISBN{979-8-4007-2596-8/2026/06}
\usepackage[svgnames]{xcolor}
\usepackage[linesnumbered,ruled,vlined]{algorithm2e}
\usepackage{graphicx}
\usepackage{subcaption}

\newtheorem{theorem}{Theorem}[section]

\newtheorem*{remark}{Remark}

\def\ie{\emph{i.e., }}
\def\eg{\emph{e.g., }}
\DeclareMathOperator{\Dom}{Dom} %

\usepackage{tcolorbox}		%
\tcbuselibrary{breakable,skins}		%
\tcbset{examplestyle/.style={
		enhanced jigsaw,	%
		colback=blue!08,	%
		colframe=blue!08,	%
		arc=2mm,
		boxrule=1pt,		%
		left=1mm,
		right=1mm,
		left skip=0mm,  %
		right skip=0mm, %
		top=0mm,		%
		bottom=1mm,		%
		breakable,		%
		parbox = false,		%
		before={\par\pagebreak[0]\vspace{.5mm}\parindent=0pt},		
		after={\par\pagebreak[0]\vspace{0mm}\parindent=0pt},		%
		bottomrule = 0mm,
		boxsep = 0mm,					%
		topsep at break=0pt,			%
		bottomsep at break=0pt,			%
		pad at break=0mm,
		pad before break=0mm,		
		pad after break=1mm,		
		bottomrule at break=0mm,
		toprule at break=0mm,		
		}}
\tcolorboxenvironment{example}{examplestyle}
\usepackage{caption}
\usepackage{placeins}
\usepackage{enumitem}

\newcommand{\answerEqyy}[1]{{\color{black} {#1}}}
\newcommand{\answerAfBc}[1]{{\color{black} {#1}}}
\newcommand{\answerHgb}[1]{{\color{black} {#1}}}
\newcommand{\answer}[1]{{\color{black} {#1}}}

\begin{document}

\title{Data Bias Mitigation under Coverage Constraints \& \\The Price of Fairness}

\author{Bruno Scarone}
\orcid{0009-0001-0155-3892}
\affiliation{%
  \institution{
  Khoury College of Computer Sciences, 
  Northeastern University}
  \city{Boston}
  \state{MA}
  \country{USA}
}
\email{scarone.b@northeastern.edu}
\author{Alfredo Viola}
\orcid{0000-0002-9518-7554}
\affiliation{%
  \institution{
  Casa de Investigadores Científicos La Comarca}
  \city{La Floresta}
  \state{Canelones}
  \country{Uruguay}
}
\email{alfredo.viola@gmail.com}
\author{Renée J. Miller}
\orcid{0000-0002-1484-4787}
\affiliation{%
  \institution{
  Cheriton School of Computer Science,
  University of Waterloo}
  \city{Waterloo}
  \state{Ontario}
  \country{Canada}
}
\email{rjmiller@uwaterloo.ca}

\renewcommand{\shortauthors}{Scarone et al.}

\begin{abstract}
Machine learning models have been shown to exhibit discriminatory outcomes or degraded performance for individuals at the intersection of multiple sensitive attributes, such as race and gender. This stems in part from two interrelated challenges: the lack of principled measures for quantifying bias (potentially intersectional), and insufficient representation of intersectional subgroups in training data. %
We extend a recent bias mitigation framework to incorporate coverage constraints that enforce sufficient representation across groups, including intersectional subgroups. Since achieving exactly zero bias for all groups may not be data efficient (meaning it may require large amounts of data), our solution trades small approximation errors in bias for greater data efficiency while satisfying coverage constraints.
We also formulate bias mitigation as an integer linear program that optimizes over all mitigation strategies, and characterize the price of fairness, the minimum data modification cost, as a function of fairness tolerance. This is essential both for legal compliance, where regulations may mandate specific fairness thresholds, and for data governance, enabling practitioners to make informed trade-offs between bias reduction and data modification (particularly, data purchasing) costs. We evaluate our techniques on publicly available datasets, demonstrating that bias mitigation via our framework preserves predictive accuracy across multiple classifiers, and that coverage constraints, while motivated by statistical considerations, are essential for preserving downstream ML performance.
\end{abstract}



\maketitle

\section{Introduction}\label{sec:intro}

Recently, there has been progress in proposing statistical measures for data bias, including intersectional bias, with the goal of providing practical measures that can be used operationally within machine learning (ML)~\cite{kanubala2025misalignment,scarone2025principled}. These measures assume a dataset with one or more sensitive attributes that define a set of sensitive groups along with a label (or class) attribute.  We build on this foundation to consider the issue of (sensitive) group coverage, a criterion  which is central to machine learning and one that is rarely considered in the statistical bias literature. We also consider the cost of bias mitigation, specifically the cost of acquiring new data.  Our work allows a data scientist to directly compare this cost with their fairness tolerance (that is, how much bias is acceptable in the data). 
In practical bias mitigation scenarios, ensuring that the resulting dataset meets certain minimum representation requirements is often as important as achieving fairness \cite{shahbazi2023representation,nargesian2021tailoring}. For instance, downstream statistical analyses or machine learning models trained on the mitigated data may require a minimum number of samples per group-label pair to produce reliable estimates. Similarly, regulatory or operational constraints may mandate that each demographic group maintains adequate representation after any data transformation. We refer to such requirements as \emph{coverage constraints}.
In the context of bias mitigation, these constraints specify, for each group-label combination in the dataset (\eg hired candidates from each demographic group), the minimum \answerAfBc{absolute count} of tuples the mitigated dataset must \answerAfBc{retain}. \answerAfBc{As we show in Section~\ref{sec:ml_evaluation}, without explicit coverage constraints, optimization-based bias mitigation can reduce datasets to the point of destroying predictive performance of ML models.} The challenge lies in finding solutions that %
achieve low bias while respecting these coverage requirements, all while minimizing the additional data needed or the number of changes to existing data.

\textbf{Motivating example.} A major challenge in data fairness is ensuring that datasets used for analysis appropriately represent relevant demographic groups \cite{nargesian2021tailoring}. 
Training data that fails to adequately represent relevant populations has led to well-documented harms across ML applications \cite{hort2024bias}. Moreover, bias often manifests differently across intersections of demographic attributes, for instance, the experiences of Black women may not be captured by examining race and gender separately~\cite{crenshaw2013demarginalizing,buolamwini2018gender}.
Consider a property management company that decides to build an ML model to screen rental applications. The model predicts whether an applicant's income exceeds a threshold deemed sufficient to afford rent, and applicants predicted to fall below this threshold are automatically rejected. To train this model, the company uses the Adult dataset \cite{misc_adult_2}, a widely used dataset containing demographic attributes from over 48,000 individuals, including a label indicating whether annual income exceeds \$50,000. Upon inspection and consultation with domain experts, analysts run a bias analysis that results in Figure~\ref{fig:intro_adult_example} (a).
\begin{figure}[t]
    \centering
    \begin{subfigure}[c]{0.48\textwidth}
        \centering
        \includegraphics[width=\textwidth]{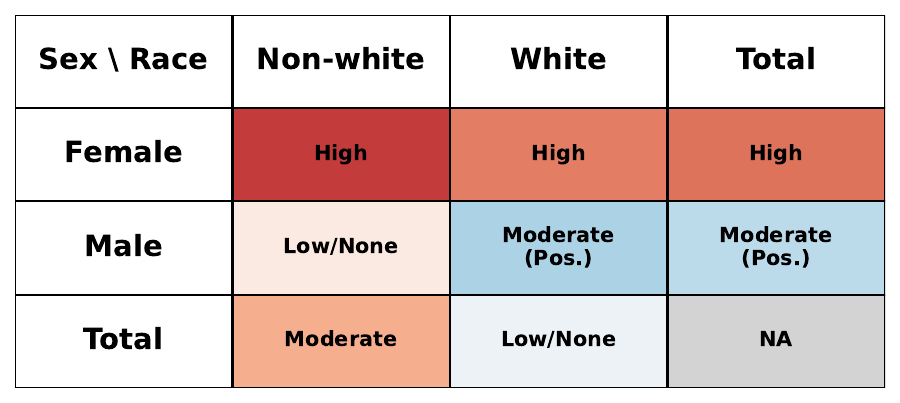}
    \end{subfigure}
    \hfill
    \begin{subfigure}[c]{0.43\textwidth}
        \centering
        \includegraphics[scale=.65]{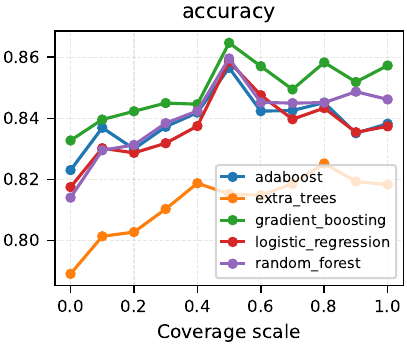}
    \end{subfigure}
    \vspace{-1mm}
    \caption{Motivating example: (a) Bias detected in training data. (b) Satisfying coverage improves ML models. \answerAfBc{Bias levels correspond to the Uniform Bias measure~\cite{scarone2025principled}, defined in Section~\ref{sec:preliminaries}. A coverage scale of $x$ requires retaining at least a fraction $x$ of the original tuples per group-label pair (Section~\ref{sec:price_of_fairness}).}}
    \Description{}
    \label{fig:intro_adult_example}
\end{figure}
The domain experts explain that the analysis reveals that data exhibits several bias patterns: while the dataset may appear relatively balanced when examining race or gender in isolation, data from Black women shows intersectional discrimination, where the combination of both race and gender produces unique disparities, and data on women shows both additive discrimination, where disadvantages accumulate independently across attributes, and intersectional effects.
A model trained on this data would likely underestimate income for these groups, systematically denying them access to housing.
The analysts wish to mitigate these biases before training their model. However, they face an additional challenge: for each demographic group, they must retain sufficient data to ensure the model learns meaningful patterns for that group. Simply removing majority-group samples until proportions balance could leave too few examples overall, while naive oversampling may introduce other distortions. Moreover, the analysts notice that training models on larger bias-mitigated datasets yields significant accuracy improvements, as shown in Figure~\ref{fig:intro_adult_example}. 
In this figure, one can see how requiring $x\%$ more training data for each group (x-axis) improves the accuracy of ML models (y-axis). 
The challenge is developing techniques that produce a mitigated dataset that simultaneously achieves low bias for all groups (even intersectional groups) while maintaining adequate coverage for all groups, ideally with minimal modifications to the dataset. In this work, the only data modifications we consider are data additions (adding new real tuples\answer{, a common practice in data integration, \eg retrieved from open data repositories, data lakes, or from data brokers~\cite{nargesian2021tailoring,chen2018my})}
or deletions (removing tuples). We address the problem of constructing  datasets that achieve both a fairness goal and are sufficiently representative of all groups (that is, satisfy coverage constraints). \answer{Moreover, permitting even modest deletions of overrepresented tuples can significantly reduce the number of additions required for fairness: for the Adult dataset, 1,059 strategic deletions reduce required additions from 4,201 to 1,599 (a 62\% reduction). This motivates why our ILP models both additions and deletions %
with potentially different costs.}

\textbf{Contributions.} 
The contributions of our work are summarized as follows.

\noindent (1) We extend the bias mitigation framework of Scarone et al.~\cite{scarone2025principled} to incorporate coverage constraints. We characterize all exact solutions, establish tight theoretical bounds on approximation errors, and design an algorithm that computes coverage-satisfying approximate solutions using the minimum amount of data for any given mitigation strategy. 
When the distributions of external data sources may be unknown, we derive \answer{sample complexity bounds based on Serfling's inequality} that guarantee accurate distribution estimates with substantially fewer samples.

\noindent (2) We formulate bias mitigation as an integer linear program (ILP) and characterize the price of fairness, the minimum data modification cost as a function of fairness tolerance, empirically validating the framework across multiple datasets. \answerHgb{This is essential for regulatory compliance, where mandates such as the EU AI Act require organizations to monitor bias in high-risk AI systems, effectively necessitating fairness thresholds, and for data governance, enabling practitioners to quantify the resource investment needed to meet a chosen standard.} %
Unlike prior work, our ILP finds the globally optimal strategy rather than fixing one in advance.

\noindent (3) We demonstrate empirically that bias mitigation via our ILP framework preserves predictive accuracy across multiple classifiers, and that coverage constraints, while motivated by statistical considerations, are essential for preventing  bias mitigation from shrinking or skewing the data to the point of degrading downstream ML performance. 
Note that our method explicitly handles non-binary label spaces (and multiple non-binary sensitive attributes), a setting underexplored in the bias mitigation literature. 
To the best of our knowledge, this is the first study to formally decouple fairness requirements from data representation guarantees in pre-processing bias mitigation, demonstrating that explicit coverage constraints are necessary to prevent pathological dataset reduction. We also show that fairness constraints alone, contrary to prevailing assumptions about fairness-accuracy trade-offs, need not degrade predictive performance.
\noindent
The rest of this paper is organized as follows. Section~\ref{sec:related_work} covers related work. Section~\ref{sec:preliminaries} introduces notation and datasets. Section~\ref{sec:UB_coverage_algorithm} presents our coverage-satisfying mitigation algorithm. Section~\ref{sec:price_of_fairness} formulates bias mitigation as an ILP and characterizes the price of fairness. Section~\ref{sec:ml_evaluation} evaluates how mitigation affects ML performance, and Section~\ref{sec:conclusions} concludes.

\section{Related work}\label{sec:related_work}

\answerAfBc{Our framework is built on a group fairness measure (Uniform Bias~\cite{scarone2025principled}) that handles intersectional, non-binary settings, and operates at the data level, following data-centric AI principles. %
Thus, we first discuss bias measures and the choice of group fairness, then intersectionality, then representation bias and
data-centric approaches for coverage.}
Researchers have proposed several measures~\cite{mehrabi2021survey, hort2024bias} to quantify bias. These measures can be classified into individual and group (statistical) measures. 
Individual fairness~\cite{dwork_fairness_through_aw} requires that similar individuals receive similar treatment, though fundamental objections have been raised regarding its assumptions and operationalization~\cite{fleisher2021individual}.
Group fairness focuses on treating different demographic groups equally.
Fleisher~\cite{fleisher2023algorithmic} argues that group fairness criteria provide defeasible evidence for system fairness rather than logical conditions. %
Hertweck et al.~\cite{hertweck2021moral} argue that metric selection requires engagement with moral philosophy, and Friedler et al.~\cite{friedler2021possibility} show that under a structural bias worldview, group fairness becomes the only non-discriminatory approach.
The heterogeneous nature of these measures makes comparison difficult~\cite{lum2022debiasing,quinonero2023disentangling}. Žliobaite~\cite{vzliobaite2017measuring} categorizes statistical measures, providing foundational notions of dataset fairness, while Yeh et al.~\cite{10.1145/3630106.3658922} establish guidelines for ratio- and difference-based measure families.
Scarone et al.~\cite{scarone2025principled} introduce Uniform Bias (UB), an interpretable measure that handles multi-valued sensitive attributes and labels, addressing gaps identified by recent surveys~\cite{hort2024bias}.
Crenshaw introduced intersectionality~\cite{crenshaw2013demarginalizing}, establishing that discrimination against Black women exceeds the sum of racism and sexism. In data fairness, datasets can satisfy group fairness for gender and race separately while failing at their intersections~\cite{kearns2019anempiricalstudy}. Algorithmic approaches include fairness gerrymandering~\cite{kearns2018preventing} and multicalibration~\cite{hebert2018multicalibration}. Buolamwini and Gebru~\cite{buolamwini2018gender} empirically showed intersectional bias in facial recognition. %
Foulds et al.~\cite{foulds2020intersectional} propose bounding classifier outcome ratios across intersectional groups, while Wang et al.~\cite{wang2022towards} note that fair ML has historically focused on single binary attributes. Kanubala and Valera~\cite{kanubala2025misalignment} consider intersectionality only for two binary sensitive attributes and a binary label, without coverage constraints. 
We address multiple non-binary sensitive attributes, any number of labels, and coverage constraints for intersectional groups.
Data-driven algorithms are only as good as the data they work with \cite{barocas2016big,shahbazi2023representation}. Shahbazi et al.~\cite{shahbazi2023representation} survey techniques for identifying and resolving representation bias, noting that systems may fail for underrepresented subgroups even when performing well overall. While prior work has examined how removing data patterns affects bias \cite{pradhan2022interpretable,shahbazi2023representation}, these approaches implicitly sacrifice coverage to reduce bias. We address the converse: quantifying the impact on fairness when coverage must be preserved.
The data-centric AI (DCAI) paradigm~\cite{jarrahi2023principles} emphasizes systematic improvement of data fit, including adequate coverage and absence of biases. Chen et al.~\cite{chen2018my} argue that unfairness from inadequate sample sizes should be addressed through data collection rather than model constraints, finding that additional samples often suffice to improve fairness without sacrificing accuracy.
Nargesian et al.~\cite{nargesian2021tailoring} propose integrating data from multiple sources cost-effectively such that subgroups meet user-specified count distributions for resolving insufficient representation. Kleinberg et al.~\cite{kleinberg2017inherent} prove that when there is an unequal base rate in data, satisfying different fairness measures simultaneously becomes impossible, connecting representation bias directly to fairness impossibility theorems. When adding data is infeasible, alternatives include informing users about representation bias or identifying the fewest additional points needed to cover uncovered spaces~\cite{asudeh2019assessing}. 
We extend the bias mitigation framework of Scarone et al.~\cite{scarone2025principled} to incorporate coverage constraints, addressing the representation bias concerns raised by Shahbazi et al.~\cite{shahbazi2023representation} and Chen et al.~\cite{chen2018my}. %
Our work connects the DCAI principles of adequate coverage~\cite{jarrahi2023principles} with principled bias mitigation.
Unlike the work of Nargesian et al.~\cite{nargesian2021tailoring}, we identify the data distribution required for a given fairness goal (the user does not need to provide it) and precisely characterize the number of data operations (such as data additions or deletions) needed to achieve it. For bias mitigation in the known source distribution setting, our algorithm efficiently computes coverage-satisfying approximate solutions using the minimum amount of data for a fixed mitigation strategy when both additions and deletions have the same cost. For the general case with arbitrary costs and optimization budgets, we formulate bias mitigation as an ILP that finds globally optimal solutions for a given fairness tolerance (meaning a bias bound). Crucially, while both prior work~\cite{scarone2025principled} and our closed-form algorithm fix a single mitigation strategy, our ILP optimizes simultaneously over all possible strategies, identifying the globally optimal combination of additions and deletions across all group-label pairs for a given objective. Unlike optimization approaches that constrain model training~\cite{zafar2017fairness}, our ILP operates at the data level. 
We also empirically demonstrate that the ILP framework preserves predictive accuracy across multiple ML classifiers. Finally, in the unknown source distribution setting, we propose an efficient sampling scheme with theoretical guarantees verified on real-world datasets.

\section{Preliminaries}\label{sec:preliminaries}
We consider the classical setting of algorithmic fairness where we have a classification task (\eg hiring candidates), in which tuples have a label (or class) attribute $Y$ and multiple sensitive attributes $\mathbf{S} = \langle S_1, \ldots, S_k \rangle$. Both $Y$ and each $S_i$ may take non-binary values.
We denote variables (\ie dataset attributes) by uppercase letters $X, Y, Z$; their values by lowercase letters $x, y, z$; and vectors of variables (or values) using boldface ($\mathbf{X}$ or $\mathbf{x}$). The domain of variable $X$ is $\Dom(X)$, and the domain of a vector of variables is $\Dom(\mathbf{X}) = \prod_{X \in \mathbf{X}} \Dom(X)$.

When analyzing bias across multiple sensitive attributes, we need notation to specify groups at varying levels of granularity. Building on Scarone et al. \cite{scarone2025principled}, we denote a group as a $k$-tuple over the $k$ sensitive attributes where each entry is either a specific value or the wildcard symbol $*$ (read as ``all'').\footnote{This follows a common convention in the multidimensional data analysis literature (OLAP), where the symbol $\top$ is sometimes used with the same meaning as our $*$ symbol.} A wildcard indicates marginalization over that attribute. For example, with $\mathbf{S} = \langle G, R, A \rangle$ representing gender, race, and age: $\langle \text{male}, *, * \rangle$ refers to all male individuals regardless of race or age; $\langle \text{male}, \text{white}, * \rangle$ refers to white males across all ages; and $\langle \text{male}, \text{white}, \text{young} \rangle$ is a fully specified group. %
A group with one or more wildcards is called an {\em aggregate group}. A group with no wildcards is called \emph{fully specified}. %
The fully wildcarded tuple $\langle *, \ldots, * \rangle$ represents the entire population. When computing the domain of a vector with wildcards, we ignore those entries (\eg the domain of $\langle \text{Gender},*\rangle$ is the same as the domain of $\langle \text{Gender}\rangle$). %
Let $n$ denote the total number of tuples in the dataset. For a group $\mathbf{s} \in \Dom(\mathbf{S})$ and label value $y \in \Dom(Y)$, we define: $[\mathbf{s}y]$ as the number of tuples with sensitive attributes $\mathbf{s}$ and label $y$; $\mathbf{s} = \sum_{y} [\mathbf{s}y]$ which is the number of tuples in group $\mathbf{s}$ and analogously $y = \sum_{\mathbf{s}} [\mathbf{s}y]$ representing the number of tuples with label $y$. \answerEqyy{Following Scarone et al.~\cite{scarone2025principled}, to simplify notation, when a value such as $\mathbf{s}$ or $y$ appears in a formula, it
denotes the count of tuples with that value (\ie  $|\mathbf{s}|$ and $|y|$ respectively). When clear from context, we drop brackets and wildcards, \eg
$\langle \text{male}, * \rangle$ is written as $\text{male}$.} Note that \textbf{s} may be a fully specified group or an aggregate group (such as all men). The corresponding relative frequencies are $f_{\mathbf{s}, y} = [\mathbf{s}y] / \mathbf{s}$ (the proportion of group $\mathbf{s}$ with label $y$) and $f_{y} = y / n$ (the global proportion with label $y$). 

Our bias mitigation framework is built upon the \emph{Uniform Bias} (UB) measure~\cite{scarone2025principled}. For a group $\mathbf{s}$ and label $y$, UB is defined as: $b_{\mathbf{s}, y} = 1 - f_{\mathbf{s}, y}/f_{y}$. This measure ranges from $-\infty$ to $1$. When $b_{\mathbf{s}, y} =0$, the group \textbf{s} is unbiased w.r.t. label $y$, which happens when the proportion of label $y$ within group $\textbf{s}$ matches the global proportion ($f_{\textbf{s},y}=f_y$). When $\textbf{s}=0$ or $y=0$, then $b_{\mathbf{s}, y} =0$. Intuitively, $b_{\mathbf{s}, y} >0$  indicates underrepresentation of label $y$ in group $\mathbf{s}$ (negative bias); and $b_{\mathbf{s}, y} <0$ indicates overrepresentation (positive bias). We refer the reader to Scarone et al.~\cite{scarone2025principled} for a complete treatment of the measure's properties and comparison to existing metrics. We will sometimes require rounding of a real value $x$ to the nearest larger integer $\lceil x \rceil$, or to the nearest integer $\lfloor x\rceil$.

\subsection{Data}
\label{sec:data}

We now briefly outline the datasets used in our examples and empirical evaluation. We employ three datasets that are \answerHgb{standard benchmarks} %
in algorithmic fairness research \cite{hort2024bias} \answerHgb{that enable comparison with prior work}: COMPAS~\cite{ds_compas}, Adult~\cite{misc_adult_2}, and Default~\cite{misc_default_of_credit_card_clients_350} (Table~\ref{tab:datasets}). These datasets vary in size, number of intersectional groups, label or class balance, and degree of initial bias, enabling us to assess whether observed patterns generalize across different data characteristics. The group-label counts and initial Uniform Bias values for each dataset are included in Appendix~\ref{app:datasets}.

The COMPAS dataset contains criminal history and demographic information collected by ProPublica in their analysis of the COMPAS recidivism prediction instrument. We use the raw COMPAS scores with the ternary risk label (Low, Medium, High). 
The Adult dataset, drawn from the 1994 U.S.\ Census, is a standard benchmark for income prediction. The task is to predict whether an individual's annual income exceeds \$50,000. %
The Default of Credit Card Clients dataset contains payment and demographic information for credit card holders in Taiwan. The prediction task is whether a client will default on their payment in the following month. \answerHgb{The Default dataset uses sex and
education level rather than race as sensitive attributes, demonstrating that our framework applies to any categorical grouping, \ie it operates on group-label counts and is agnostic to the semantics of the attributes.} %

\begin{table}[t]
\centering
\begin{tabular}{lrrrcc}
\toprule
\textbf{Dataset} & \textbf{n} & \textbf{Groups} & \textbf{Labels} & \textbf{Sensitive Attributes} & \textbf{\% Positive} \\
\midrule
COMPAS  & 60,798 & 4 & 3 & sex, race (binary) & N/A (ternary)   \\
Adult   & 48,842 & 4 & 2 & sex, race (binary) & 23.9\% \\
Default & 30,000 & 8 & 2 & sex, education     & 22.1\% \\
\bottomrule
\end{tabular}
\caption{Summary of datasets. The ``Groups'' column refers to the number of fully specified groups.
}
\label{tab:datasets}
\end{table}

\section{Mitigation under coverage constraints}
\label{sec:UB_coverage_algorithm}
In this section, we extend the bias mitigation framework of Scarone et al.~\cite{scarone2025principled} to incorporate coverage constraints.\footnote{\answerHgb{We illustrate the closed-form solutions using the COMPAS dataset; these solutions apply to any dataset given its group-label counts.}} In the rest of the paper we will denote the number of tuples to be added or deleted from a fully specified group $\textbf{s}$ with label $y$ by $\Delta[\textbf{s}y]$~\cite{scarone2025principled}.

\subsection{Mitigation strategies and solutions}

A fairness goal specifies the target bias for each group-label pair  that we wish to achieve. We assume a target of zero bias; this can be relaxed to non-zero targets as shown in prior work~\cite{scarone2025principled} (see Appendix~\ref{app:nonzero_bias_condition}).
A mitigation {\em solution}, \ie the number of tuples to be added or deleted from each group-label pair to achieve a desired fairness goal ($\{\Delta[\textbf{s}y]\}$), is obtained by solving the linear system given by the equations 
\begin{equation}\label{eq:UB_mitigation_linear_system}
    \Delta[\textbf{s}y] = %
    -[\textbf{s}y] + y/y_i \cdot ([\textbf{s}y_i] + \Delta[\textbf{s}y_i])%
\end{equation}
for every group-label pair $(\textbf{s},y)$. Here, $y_i$ is a user-chosen reference label for each group $\mathbf{s}$. The user provides $\Delta[\mathbf{s}y_i]$ directly, the number of additions or deletions for $(\mathbf{s}, y_i)$, and the remaining $\Delta[\mathbf{s}y]$ values are uniquely determined\footnote{%
The system is underspecified and the selection of $y_i$ determines its free variables. Note that when $y=y_i$, Equation~\ref{eq:UB_mitigation_linear_system} becomes trivial.} to satisfy the fairness goal. We call this choice of reference labels a \emph{mitigation strategy}. We say that a solution is {\em valid} if when $\Delta[\textbf{s}y] < 0$ then $|\Delta[\textbf{s}y]|\leq[\textbf{s}y]$ (we do not delete more tuples than available).
\begin{example}\label{ex:mitigation_toy}
    Suppose we have a binary gender attribute with values $m$ and $w$, and a binary label with values $T$ and $F$, with the following group-label counts: $mT=3$, $mF=1$, $wT=1$ and $wF=2$. The linear system is given by four equations.
    To fully specify it, we also need to select a mitigation strategy, \ie one reference label $y_i$ per group.
    There are two options for picking $y_i$ (meaning two labels) for each of the two groups, so four possible mitigation strategies. If for group $\textbf{s}=\answerEqyy{\langle m,*\rangle}$ we select $y_i=T$ and for $\textbf{s}=\answerEqyy{\langle w,*\rangle}$ we take $y_i=F$, the resulting system (for this strategy) is: %
    \[\begin{cases}
        \Delta [mF] &= -[mF]+F/T\cdot([mT]+\Delta[mT]) = -1+3/4 \cdot (3+\Delta[mT])\\
        \Delta [wT] &= -[wT]+T/F\cdot([wF]+\Delta[wF]) = -1 + 4/3 \cdot(2+\Delta[wF])
    \end{cases}\]
    Both $\Delta [mT]$ and $\Delta [wF]$ are the free variables of the system, meaning that assigning them a value, fully determines a solution. Here, for example, if $\Delta [mT]=1$ and $\Delta [wF]=1$, then $\Delta [mF]=2$ and $\Delta[wT]=3$ completes the bias mitigation solution. The resulting dataset is unbiased, since both groups have four positive and three negative tuples. This solution is valid because it does not delete tuples.
\end{example}

Scarone et al.~\cite{scarone2025principled} suggest always setting $i=\arg\max_{y_j}[\textbf{s}y_j]/y_j$ which ensures that the mitigation strategy includes additions only (hence $\Delta[\textbf{s}y]\geq0$ for all $(\textbf{s},y)$). However, in our work, we permit deletions and we will consider new principled ways of comparing different mitigation strategies in Section~\ref{sec:price_of_fairness}.

We start by determining all exact solutions,\ie those achieving zero Uniform Bias.
Solving the system introduced earlier, guarantees that the mitigation achieves the desired bias level. But it may yield non-integral solutions. For $\Delta [\textbf{s}y]$ to be an integer, we need $\Delta [\textbf{s}y_i]$ to be chosen such that $y/y_i \cdot ([\textbf{s}y_i] + \Delta[\textbf{s}y_i])$ is an integer as well. This is equivalent to requiring $y\cdot ([sy_i] + \Delta[\textbf{s}y_i])$ to be a multiple of $y_i$, which means\footnote{We assume $y/y_i$ is in its most simplified form, if not, it suffices to divide both by their greatest common divisor.} that $\Delta[\textbf{s}y_i] +[sy_i] =  k\cdot y_i$ for an integer $k$. Adding this constraint into our Equation~\ref{eq:UB_mitigation_linear_system}, we get $\Delta[\textbf{s}y] = -[\textbf{s}y] + k\cdot y$ for an integer $k$.  If we consider additions only (no deletions), the minimal (non-empty) exact solution is when $k=1$. Note that since $[\textbf{s}y]\leq y$ when $k=1$, $\Delta [\textbf{s}y]$ is non-negative for all groups and labels (\ie no deletions). If we were to allow $k=0$, then $\Delta[\textbf{s}y]=-[\textbf{s}y]$, so the entire table disappears, something that is trivially fair but uninteresting.

\begin{example}\label{ex:exact_compas}
    For the COMPAS dataset we have three labels:  L, M, and H. 
    So $\Delta\textbf{s} = \Delta[\textbf{s}H] + \Delta[\textbf{s}M]+\Delta[\textbf{s}L]$, and the resulting frequency after mitigation with $\Delta[\textbf{s}y] = -[\textbf{s}y] + y$, denoted $f^{\text{new}}_{\textbf{s},y}$, is
    \begin{align*}
        f^{\text{new}}_{\textbf{s},y} &= \frac{[\textbf{s}y]+\Delta[\textbf{s}y]}{\textbf{s}+\Delta\textbf{s}} = \frac{[\textbf{s}y]-[\textbf{s}y]+y}{\textbf{s}+\Delta[\textbf{s}H] + \Delta[\textbf{s}M]+\Delta[\textbf{s}L]} = \frac{y}{\textbf{s}-[\textbf{s}H]+H-[\textbf{s}M]+M-[\textbf{s}L]+L} %
        = \frac{y}{n} 
        =f_y
    \end{align*}
    Which means that $\Delta[\textbf{s}y]=-[\textbf{s}y]+y$ is a solution (the minimal one, for $k=1$), \ie produces an unbiased dataset. For the COMPAS dataset, we have four fully specified groups: (Non-White, women), (Non-White, men), (White, women) and (White, men). The smallest exact solution requires $182,394$ tuple additions. White men with a high risk score require the least additions ($5,550$) and White females with a low score the most additions ($37,328$), the full solution is included in Appendix~\ref{app:compas_coverage}. Notice that this solution provides fairness (that is, is unbiased for all twelve group-label pairs). However, it requires what a data scientist may consider a prohibitively large amount of additional data. We consider how to reduce this amount in a principled way in Section~\ref{sec:approx}.
\end{example}
\subsection{Coverage constraints}\label{sec:coverage} 

A \emph{coverage constraint} specifies the minimum \answerAfBc{absolute count} of tuples that a mitigated dataset must \answerAfBc{retain}. Formally, given a coverage threshold $m_{\textbf{s},y} \geq 1$ for each fully-specified group-label pair $(\textbf{s},y)$, we require that the mitigated dataset satisfies $[\textbf{s}y]+\Delta[\textbf{s}y]\geq m_{\textbf{s},y}$. A solution is {\em valid} with respect to coverage constraints if each group-label pair $(\textbf{s},y)$ retains at least $m_{\textbf{s},y}$ tuples.
\begin{example}
    The solution in Example~\ref{ex:mitigation_toy} is valid for the constraint $m_{\textbf{s},y}=3$ for every group-label $(\textbf{s},y)$, but invalid for the constraint $m_{\textbf{s},y}=4$, since in the mitigated data there are only three tuples for $s=m$ with label $F$.
\end{example}
The exact solutions satisfying the coverage constraint with $m_{\textbf{s},y}$ are given by $k^*\geq \lceil m_{\textbf{s},y}/y\rceil$, the minimal exact solution (that is, the most data-efficient exact solution) is the one satisfying $k^*= \lceil m_{\textbf{s},y}/y\rceil$. 

\subsection{Approximate solutions under coverage constraints}\label{sec:approx} 
Notice that exact solutions can be very large, and by permitting approximation (a small amount of bias), we can find much more data-efficient (smaller) mitigation solutions.
Scarone et al.~\cite{scarone2025principled} proposed applying rounding to Equation~\ref{eq:UB_mitigation_linear_system}, which introduces an error in the bias mitigation.  In our work, we provide a more rigorous analysis of such approximation error. Our analysis permits the extension of approximation to mitigation settings with coverage constraints.

Our key insight is to leverage the fact that the user can pick the number of operations for $(\textbf{s},y_i)$ in the mitigation. Thus, we can determine the minimum value of $\Delta[\textbf{s}y_i]$ for the coverage constraint to be satisfied. 
A coverage constraint is satisfied when $\Delta[\textbf{s}y_i]\geq\frac{y_i}{y}m_{\textbf{s},y}-[\textbf{s}y_i]$. Since we have one label $y_i$ per group $\textbf{s}$, we need to consider the constraints for all the labels. Thus, the smallest integral value to guarantee the coverage constraint is given by $\Delta [\textbf{s}y_i] = \max_{y}\Big\{\Big\lceil \frac{y_i}{y}m_{\textbf{s},y}-[\textbf{s}y_i]\Big\rceil\Big\}$ for each group-label pair.\footnote{If we want to enforce $\Delta [\textbf{s}y_i]\geq 0$, as was done in Scarone et al.~\cite{scarone2025principled}, it suffices to add $\{0\}$ before computing the maximum.}
After computing $\Delta [\textbf{s}y_i]$ in this way, we can use Equation~\ref{eq:UB_mitigation_linear_system} to compute the values of $\Delta[\textbf{s}y]$, and rounding them produces the output of our new mitigation algorithm satisfying the coverage constraints. 
The same approach can be used to accommodate upper bound constraints of the form $[\textbf{s}y]+\Delta [\textbf{s}y] \leq M_{\textbf{s},y}$, for both exact and approximate solutions (Appendix~\ref{app:upper_bound_constraints}).

\begin{example}
    Continuing Example~\ref{ex:exact_compas}, applying our mitigation algorithm to the COMPAS dataset with coverage constraint $m_{\textbf{s},y}=1,000$ for every $\textbf{s}$ and $y$, and free variable selection $i=\arg\max_j [\textbf{s}y_j]/y_j$ 
    we achieve an average resulting %
    bias of $0.00014$. Complete results are provided in Appendix~\ref{app:compas_coverage}. 
\end{example}

Now, we analyze the quality of the approximate solutions of the coverage-satisfying mitigation algorithm. Earlier, we derived a closed form for all exact solutions to the problem. These solutions generally require large amounts of data, since $|y|$ is typically large. Given that acquiring suitable data is costly, finding  good approximate solutions is of utmost importance. 
When rounding the values produced by Equation~\ref{eq:UB_mitigation_linear_system}, we introduce a {\em rounding error}, which is always in $[0,1)$. In general, the smaller the sizes of the groups, the more significant this %
error is. Ultimately, we are not interested in the rounding error that occurs when estimating $\Delta [\textbf{s}y]$, but on how much this error affects the frequencies, %
\ie {\em frequency error}: the difference between the estimated frequency for a group-label and the objective, $|f^{\text{ap}}_{\textbf{s},y}-f_y|$. 
We care about the frequency error as it determines the (non-zero) bias remaining in the approximate solution. We know that when using an exact solution, denoted by $\Delta^{\text{ex}}_{\textbf{s},y}$, we get a resulting frequency $f^{\text{ex}}_{\textbf{s},y}$ that satisfies $f^{\text{ex}}_{\textbf{s},y}=f_y$ and thus there is no error (and no bias). In the worst case, using an approximate solution $\Delta^{\text{ap}}_{\textbf{s},y}$, one gets a rounding error of one for the group-label and of $k$ for the group (summing over all labels) and thus in the worst case the error is at most $E=\frac{[\textbf{s}y]+\Delta^{\text{ex}}_{\textbf{s},y}+1}{\textbf{s}+\Delta^{\text{ex}}_{\textbf{s}}+k}-\frac{[\textbf{s}y]+\Delta^{\text{ex}}_{\textbf{s},y}}{\textbf{s}+\Delta^{\text{ex}}_{\textbf{s}}}$. 

\begin{theorem}[Error bounds for approximate bias mitigation]\label{thm:approx_errors_bounds}
Our approximate mitigation solutions satisfy the following bounds on the maximal frequency error:
$
\frac{1-k}{\textbf{s}+\Delta^{\text{ex}}_{\textbf{s}}+k} \;\leq\; E \;\leq\; \frac{1}{\textbf{s}+\Delta^{\text{ex}}_{\textbf{s}}+k}
$. Both bounds are tight.
\end{theorem}

\begin{example}
    For example, with $m=\textbf{s}+\Delta^{\text{ex}}_{\textbf{s}}\geq 1100$ and $k\leq 10$, the upper bound becomes $\frac{1}{m+k}\leq \frac{1}{1100}\approx0.0009$
    and the lower bound:$\frac{1-k}{m+k}\geq \frac{1-10}{1100+10}=\frac{-9}{1110}\approx-0.0081$. So, the error is approximately bounded in $[-0.008,0.001]$.
\end{example}
When $m=\textbf{s}+\Delta_\textbf{s}^{\text{ex}}\gg k$, \ie the number of tuples from group $\textbf{s}$ in an exact solution is much bigger than the number of labels, which is the case in practical applications, the bounds of Theorem~\ref{thm:approx_errors_bounds_general_deltas} behave like $[-\frac{k}{m},\frac{1}{m}]$, both being 
$O(1/m)=O(1/\textbf{s}+\Delta^{\text{ex}}_{\textbf{s}})$
, which is remarkably small. And importantly, this small frequency error, translates to a very small bias that is introduced by approximation into a mitigated data solution.
\subsection{\answer{Sample complexity for external source distribution estimation}}\label{sec:coverage_source_estimation}%

\answer{This section addresses a problem distinct from the coverage constraints of Section~\ref{sec:coverage}: whereas coverage constraints guarantee minimum representation in the \emph{output} dataset, here we determine how many samples are needed to reliably estimate an external source's distribution \emph{before} optimization. This per-source estimation is
necessary to correctly parameterize any bias mitigation algorithm, %
which requires knowing the distribution of available (external) data.}
Consider a data scientist who wishes to mitigate bias by incorporating samples from  external sources. If the external source's distribution were known, one could directly apply the techniques described in Section~\ref{sec:approx} or Section~\ref{sec:price_of_fairness}. In practice, however, external sources (\eg data brokers) may not disclose distributions. So the data scientist must estimate this distribution from samples, raising the question: {\em how many samples suffice?} Without principled guidance, one risks either undersampling (producing unreliable estimates) or oversampling (wasting resources that may be costly or limited). We address this by deriving coverage thresholds from statistical concentration bounds, specifying the minimum samples needed to guarantee accurate distribution estimates with high probability. Most notably, the sample complexity is independent of the dataset size, depending only logarithmically on the number of group-label pairs and the confidence parameter.

Consider uniformly sampling $n$ tuples from a  table of $N$ tuples, partitioned into $c$ distinct group-label pairs. Let $p_i$ denote the true proportion of a given group-label $i$ in the table, where $\sum_{i=1}^c p_i = 1$. We draw a uniform random sample of $n$ tuples without replacement. For each group-label $i$, let $X_i$ denote the count of tuples from  $i$ in the sample, with $\sum_{i=1}^c X_i = n$, and let $\bar{X}_i = X_i / n$ denote the corresponding sample proportion. To analyze the distribution of $X_i$, we introduce indicator random variables $Y_1^{(i)}, \dots, Y_n^{(i)}$, where $Y_j^{(i)} = 1$ if the $j$-th sampled tuple belongs to group $i$ and $Y_j^{(i)} = 0$ otherwise, so that $X_i = \sum_{j=1}^n Y_j^{(i)}$. Since we sample without replacement, the indicator random variables $Y_1^{(i)}, \dots, Y_n^{(i)}$ are not independent. We therefore apply Serfling's inequality~\cite{serfling1974probability, bardenet2015concentration}, which extends Hoeffding's bound to sampling without replacement from finite populations. We include the formal statements of these classic results and other supplementary materials in Appendix~\ref{app:coverage_LB_estimation}.
Applying Serfling's inequality to the sample proportion $\bar{X}_i = X_i / n$, we obtain for any $\epsilon > 0$ a bound on the error probability: $\Pr[|\bar{X}_i - p_i| \geq \epsilon] \leq 2\exp\left(-\frac{2n\epsilon^2}{1 - \frac{n-1}{N}}\right)$, where $E[\bar{X}_i] = p_i$ holds. Now, if we look at the event $\max_{i=1,\dots,c} |\bar{X}_i - p_i| \geq \epsilon$, this is equivalent to the union of the events $\{|\bar{X}_i - p_i|\}$. Using a union bound, we get $\Pr\left[\max_{i=1,\dots,c} |\bar{X}_i - p_i| \geq \epsilon\right] \leq 2c\exp\left(-\frac{2n\epsilon^2}{1 - \frac{n-1}{N}}\right)$, which upper bounds the total probability of error of our estimation. Setting the upper bound equal to $\delta$ and solving for $n$, we obtain $n =  \frac{(N + 1)\ln(2c/\delta)}{\ln(2c/\delta) + 2\epsilon^2 N}$. This is the smallest sample size required to achieve $\epsilon$-accuracy with error probability at most $\delta$. The sample complexity is $O(\log(c/\delta)/\epsilon^2)$, independent of the dataset size $N$, making the approach scalable to large datasets (see Appendix~\ref{app:sample-complexity} for the derivation). %

We implement the sampling procedure with $\epsilon = 0.05$ and $\delta = 0.05$ for the COMPAS dataset. Using our methods, the required sample size is $n = 1210$, which is $2\%$ of the total dataset. We repeat the experiment $100$ times. For each repetition, we compute the estimation error $|\bar{X}_i - p_i|$ for every group-label pair $i$. All observed errors are well below the theoretical bound of $\epsilon = 5\%$, with the average error across  groups being $0.52\%$ and the largest $1.04\%$ (for non-white men with a medium risk score). The full list is included in Appendix~\ref{app:coverage_LB_estimation}.
\section{The price of fairness}\label{sec:price_of_fairness}

Achieving fairness in data on a biased dataset requires modification, whether through tuple additions, deletions, or both, and these modifications come at a cost. A natural question arises: \emph{how much does fairness cost?} More precisely, what is the minimum amount of data modification required to achieve a given level of fairness, and how does this cost vary as fairness requirements become stricter?\footnote{\answerHgb{The ``price of fairness'' quantifies the cost of
\emph{achieving} a fairness standard, not the cost of tolerating bias or discrimination. When a regulation specifies a fairness standard (typically operationalized as a value of~$\epsilon$), our framework tells practitioners the (minimum) data modifications required to satisfy it.}} The algorithm developed in Section~\ref{sec:UB_coverage_algorithm} answers this question for a fixed mitigation strategy and prior work also only considers fixed strategies~\cite{scarone2025principled}.\footnote{Recall that a mitigation strategy is a choice of reference label $y_i$ (free variable) per group $\textbf{s}$ in the linear system of equations (Section~\ref{sec:coverage}).}  Our algorithm permits the computation of coverage-satisfying approximate solutions using a minimum amount of data for one choice of $y_i$ per group. However, different mitigation strategies may yield different costs, and practitioners may wish to optimize over all possible strategies simultaneously, incorporate budget constraints, or explore trade-offs between competing objectives such as minimizing additions versus minimizing deletions. To address these needs, we formulate bias mitigation as an integer linear program (ILP) that finds globally optimal solutions across all strategies. The %
structure of the ILP is: %
\begin{align}
    \min \quad & f(\{[\textbf{s}y]\},\{\Delta[\textbf{s}y]\}) & & \text{(minimize objective)} \\
    \text{s.t.} \quad  &  |f^{\text{new}}_{\textbf{s},y}-f_y|\leq\epsilon & \forall \textbf{s}, y \quad & \text{(approximate fairness)} \\
    & [\textbf{s}y]+\Delta[\textbf{s}y] \geq m_{\textbf{s},y} & \forall \textbf{s}, y \quad & \text{(coverage guarantee)} \\ %
    & \sum_{\textbf{s},y} \text{cost}(c_a,c_d,\Delta [\textbf{s}y])\leq B & \quad & \text{(costs and budget)} %
\end{align}
\noindent
As before, for each %
fully-specified group $\textbf{s}$ and outcome $y$, $[\textbf{s}y]$ denotes the original count and $\Delta[\textbf{s}y] \in \mathbb{Z}$ the decision variable representing the change to that group-label. %
The objective $f$ minimizes a function of the resulting counts, \eg  total dataset size or minimal change to the data. The fairness constraint bounds the deviation between each group-label's outcome distribution and the global target distribution $f_y$, parameterized by a \textbf{fairness tolerance} $\epsilon$: when $\epsilon = 0$, we require exact fairness (in this case it coincides with demographic parity); larger values permit approximate fairness in exchange for fewer data changes.\footnote{\answerAfBc{We bound $|f_{\mathbf{s},y} - f_y|$ rather than UB directly ($b_{\mathbf{s},y} = 1 - f_{\mathbf{s},y}/f_y$) because the ratio form would introduce non-linearity for no practical benefit. When $|f_{\textbf{s},y} - f_y|$ is small, UB is small as well, so the two are interchangeable for practical purposes.}} We also incorporate a budget constraint that bounds the total cost of modifying the dataset, where additions and deletions may incur different costs $c_a$ and $c_d$, respectively. The cost function is piecewise linear: $c_a\cdot \Delta[\textbf{s}y]$ for additions $\Delta[\textbf{s}y] > 0$ and $c_d \cdot \Delta[\textbf{s}y]$ for deletions $\Delta[\textbf{s}y] < 0$. There are two {\em a priori} non-linear components of our ILP: the approximate fairness constraint involving a ratio of decision variables and the piecewise linear form of the cost function. We show how to linearize both in Appendix~\ref{app:ILP_linearization}.
\answer{We explore three objective functions that capture distinct practical goals.

\textbf{Minimize dataset size (\texttt{min\_size}).}  When the goal is to retain the smallest dataset that satisfies both fairness and coverage constraints, we minimize the resulting dataset size:
$\min \sum_{\mathbf{s},y} ([\mathbf{s}y] + \Delta[\mathbf{s}y])$. This objective is motivated by data minimization principles such as the GDPR's data minimization requirement (Art.~5(1)(c)), and by storage or computational efficiency.} %

\textbf{Minimize total modifications (\texttt{min\_changes}).}  When the primary concern is preserving the statistical properties of the original dataset, we minimize the total number of modifications regardless of direction: $\min \sum_{\mathbf{s},y} \bigl(\Delta^{+}_{\mathbf{s},y} + \Delta^{-}_{\mathbf{s},y}\bigr)$. Here $\Delta^{+}_{\mathbf{s},y}$ and $\Delta^{-}_{\mathbf{s},y}$ denote additions and deletions for the group-label, respectively.  This objective is motivated by data quality considerations: each modification, whether an addition or deletion, potentially distorts the original data distribution in ways beyond the intended fairness correction.
\answerEqyy{A natural question is why model deletions when the original data is available. The answer is that deletions and additions interact non-trivially: deleting overrepresented tuples shifts global proportions, lowering the target underrepresented groups must reach. As we show in Section~\ref{sec:ILP_experimental_evaluation}, for the Adult dataset with $\epsilon = 0.05$, restricting to additions only ($x = 1$) requires 4,201 additions; permitting 1,059 strategic deletions reduces this to 1,599---a 62\% reduction---with total modifications dropping from 4,201 to 2,658. Since data acquisition may be expensive or unavailable~\cite{chen2018my}, this flexibility yields substantial savings in practice. Coverage constraints (Section~\ref{sec:coverage}) formally guarantee that this flexibility cannot severely degrade data representation.}

\textbf{Minimize cost (\texttt{min\_cost}).} We minimize the total weighted cost:
$\min \sum_{\mathbf{s},y}(c_a \cdot \Delta^{+}_{\mathbf{s},y} + c_d \cdot
\Delta^{-}_{\mathbf{s},y})$. This is of practical interest because additions and deletions may incur different costs: for example, acquiring new data may be more expensive than removing existing records. This objective is the linearized cost derived in Appendix~\ref{app:ILP_linearization}.

\subsection{Experimental evaluation}\label{sec:ILP_experimental_evaluation}
The ILP formulation for bias mitigation includes three key parameters influencing the solution: the fairness tolerance $\epsilon$, the coverage thresholds $m_{\textbf{s},y}$ and the cost asymmetry $(c_a, c_d)$. We study how each parameter affects the solutions for the objectives introduced earlier. These experiments are conducted on the COMPAS, Adult, and Default datasets, which vary in size and degree of initial bias, allowing us to assess if observed patterns generalize or are dataset-specific.\footnote{The source code is available at \url{https://github.com/bscarone/data-bias-coverage}. } When all three datasets exhibit the same pattern, we present only one in the main text; complete results appear in Appendix~\ref{app:ILP_experiments}.
For the fairness tolerance experiment, we require non-empty groups (\ie $m_{\textbf{s},y}=1$ for every group-label pair) and vary the fairness tolerance $\epsilon$ from $0.01$ to the maximum difference $|f_{\textbf{s},y}-f_y|$ in the original data, the point at which no modifications are required. %
The results are shown in Figure~\ref{fig:vary_eps_compas}. 
\begin{figure}
    \centering
    \includegraphics[width=0.9\linewidth]
    {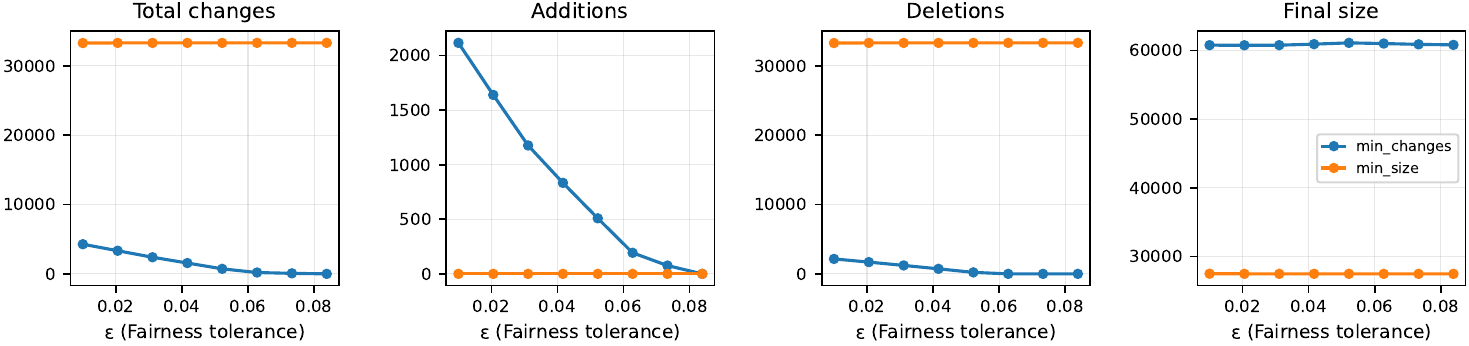}
    \caption{Effect of fairness tolerance ($\epsilon$) on mitigation solutions for the COMPAS dataset \answer{($m_{\textbf{s},y}=1$, $c_a=c_d=1$)}.}
    \Description{}
    \label{fig:vary_eps_compas}
\end{figure}
The objectives exhibit two distinct behaviors. The \texttt{min\_changes} objective produces the expected tradeoff curve: stricter fairness requirements (smaller $\epsilon$) demand more modifications, with changes decreasing monotonically as $\epsilon$ increases until reaching zero at the maximum difference of frequencies. %
This objective balances additions and deletions, using whichever combination minimizes total modifications for each group-label pair.
In contrast, %
\texttt{min\_size} pursues a deletion-only strategy, achieving zero additions across all datasets and $\epsilon$ values. 
Notably, \texttt{min\_changes} retains substantially more data than the other objectives. For instance, at $\epsilon = 0.01$, \texttt{min\_changes} preserves $48,887$ records for Adult, $60,758$ for COMPAS, and $29,991$ for Default, while \texttt{min\_size} retains only $21,788$, $27,517$, and $72$ respectively. This highlights a practical consideration: with minimal coverage constraints, objectives that do not penalize deletions aggressively discard data, potentially eliminating valuable information even when modest modifications would suffice for bias mitigation.

For the coverage experiment, we fix a fairness tolerance of $\epsilon=0.05$ and vary the coverage scale from $0$ to $1$ in increments of $0.1$. A coverage scale of $x$ sets $m_{\textbf{s},y} = \max(1, \text{round}(x \cdot [\textbf{s}y]))$ for each group-label pair, requiring retention of that fraction of the original counts. At $x=0$, this reduces to the non-empty constraint ($m_{\textbf{s},y}=1$); at $x=1$, $m_{\textbf{s},y}=[\textbf{s}y]$, thus deletion-only strategies are infeasible. The results are shown in Figure~\ref{fig:vary_coverage_adult}.
\begin{figure}
    \centering
    \includegraphics[width=0.9\linewidth]{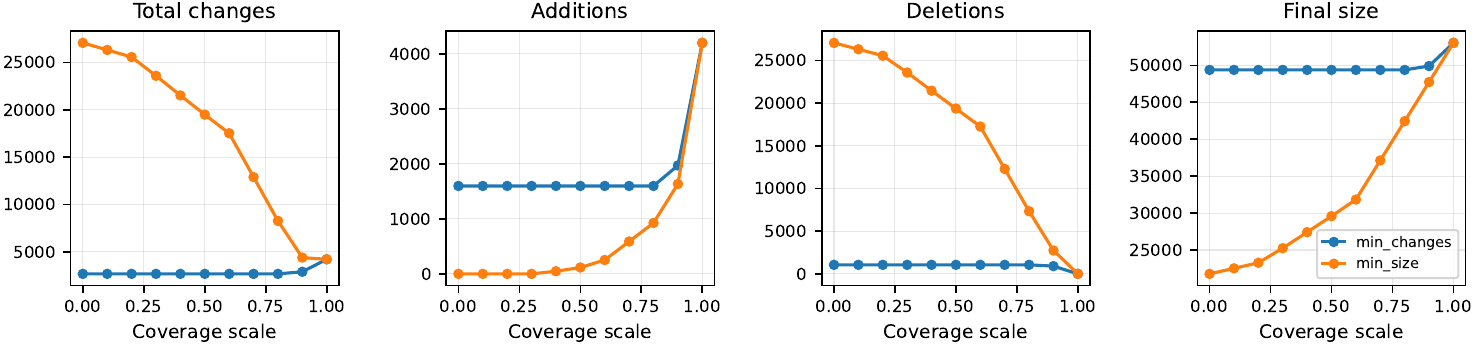}
    \caption{Effect of coverage scale on mitigation solutions for the Adult dataset \answer{($\epsilon = 0.05$, $c_a=c_d=1$)}.}
    \Description{}
    \label{fig:vary_coverage_adult}
\end{figure}
The results reveal two distinct behaviors. The \texttt{min\_changes} objective produces stable solutions: total changes remain nearly constant across coverage scales from $0$ to $0.9$, increasing only at $x=1$ when deletions are prohibited. This stability arises because \texttt{min\_changes} finds an efficient balance of additions and deletions that satisfies fairness constraints regardless of coverage requirements.
In contrast, \texttt{min\_size} is highly sensitive to coverage. At low coverage, it deletes aggressively, retaining only what coverage requires. As coverage increases, permissible deletions decrease linearly, forcing a gradual transition toward addition-based strategies. %
At $x=1$, both objectives converge to the same addition-only solution. Notably, \texttt{min\_changes} requires substantially more modifications at $x=1$ than at lower coverage scales (\eg $2{,}658$ vs.\ $4{,}201$ for Adult), demonstrating that the flexibility to delete can significantly reduce the total cost of achieving fairness.

For the cost experiment, we fix a fairness tolerance of $\epsilon=0.05$, we require non-empty groups and vary the cost ratio $c_d/c_a$ from $0.1$ to $10$, with $c_a$ fixed at $1.0$. When $c_d/c_a < 1$, deletions are cheaper than additions; when $c_d/c_a > 1$, additions are cheaper. At $c_d/c_a = 1$, minimizing cost is equivalent to \texttt{min\_changes} (with the same $B$). Results are shown in Figure~\ref{fig:vary_cost_default}. 
All three datasets exhibit a clear phase transition from deletion-heavy to addition-heavy strategies as $c_d/c_a$ increases. When deletions are cheap (low $c_d/c_a$), the optimizer removes data aggressively; when deletions are expensive (high $c_d/c_a$), it relies exclusively on additions. The transition occurs at different thresholds depending on the dataset structure: COMPAS shifts to addition-only at $c_d/c_a \geq 2$, while Adult and Default require $c_d/c_a \geq 4$. %
At intermediate ratios, the optimizer employs mixed strategies that balance additions and deletions according to their relative costs.
\begin{figure}
    \centering
    \includegraphics[width=0.9\linewidth]{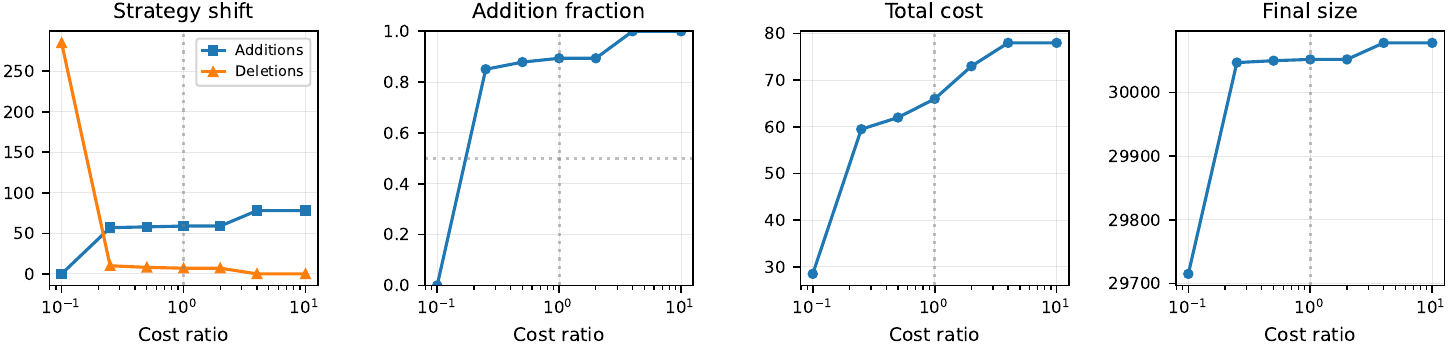}
    \caption{Effect of cost ratio on mitigation solutions for the Default dataset \answer{($\epsilon=0.05$, $m_{\textbf{s},y}=1$)}.}
    \Description{}
    \label{fig:vary_cost_default}
\end{figure}

\answer{
\textbf{Coverage constraints, objectives \& dataset structure.}
The experiments above reveal that coverage constraints and optimization objectives play complementary and distinct roles. Under \texttt{min\_changes}, the objective itself penalizes every deletion, so group sizes are implicitly preserved: coverage constraints rarely have a visible effect at scales below 1 (Figure~\ref{fig:vary_coverage_adult}). Under \texttt{min\_size}, the objective actively seeks to shrink the dataset, so without explicit coverage constraints the optimizer is free to collapse group-label counts. The Default results in Section~\ref{sec:ml_evaluation} (Figure~\ref{fig:ml_eval_vary_eps_min_size_default}) make this concrete: when fairness constraints are relaxed (higher~$\epsilon$), \texttt{min\_size} reduces the dataset to as few as 24~observations, rendering reliable predictive performance impossible. 
The severity of this effect depends on dataset structure: Default is most vulnerable because its ``Other'' education groups contain only 170--298~tuples with large initial frequency deviations ($|f_{\mathbf{s},y} - f_y| \approx 0.15$), while COMPAS is more resilient due to larger minimum group sizes (5{,}428~tuples) and smaller deviations ($|f_{\mathbf{s},y} - f_y| \leq 0.08$). The cost ratio phase transitions similarly reflect dataset structure, with COMPAS shifting to addition-only strategies at $c_d/c_a \geq 2$ versus $c_d/c_a \geq 4$ for Adult and Default. Stricter fairness constraints (smaller $\epsilon$) only prevent this type of collapse because they can \emph{inadvertently} act as implicit coverage constraints, forcing the optimizer to retain more data. This coupling is fragile: it ties data sufficiency to a fairness parameter that practitioners may wish to relax. Our explicit coverage guarantees decouple these two concerns, giving practitioners independent control over fairness tolerance and minimum representation. Coverage constraints are essential for objectives that permit aggressive deletion, providing a formal guarantee that the optimizer will maintain sufficient data representation.}
\section{ML-Based Evaluation}\label{sec:ml_evaluation}
In this section, for each experiment detailed in Section \ref{sec:price_of_fairness}, we train five ML models on the mitigated and biased versions of the datasets to evaluate how our data bias mitigation affects ML performance. 
We follow the methodology used by Scarone et al.\ \cite{scarone2025principled}, which we review next. To simulate the setting in which we have a biased dataset $T$ and wish to collect external data (\eg from an open data lake) to mitigate its biases, we partition $T$ uniformly at random into two sets: an ``initial sample'' of size $x_0$ and the remaining portion, which we treat as available external data. Since our ILP experiments evaluate multiple mitigations (one per parameter value), we select $x_0$ such that the external pool contains enough data to supply the additions required by all mitigations. The procedure for determining $x_0$ is described in Appendix~\ref{app:x0_selection}.
Before sampling, we drop identifier columns and remove rows with nulls. Before training, we standardize numeric features and one-hot encode categorical features.
We use an 80/20\% train-test split. For ML models, we use Random Forest, Gradient Boost Decision Trees (GBDT), Extra Trees, Ada Boost and Logistic Regression. We use Scikit-learn\footnote{\url{https://scikit-learn.org/stable/}, accessed 2026-01-02.} to implement our models. As before, when all three datasets exhibit the same pattern, we present only one in the main text; complete results appear in Appendix~\ref{app:ml_eval_experiments}. We report accuracy across all experiments.\footnote{While precision, recall and balanced accuracy offer interpretable error decomposition for binary classification, they require aggregation schemes %
for multi-label settings that obscure their intuitive meaning. Since our method explicitly handles non-binary label spaces, a setting underexplored in the bias mitigation literature, we prioritize metrics that remain directly comparable across binary (Adult, Default) and ternary (COMPAS) settings.}

Figure~\ref{fig:ml_eval_vary_epsilon_min_changes_compas} presents accuracy as a function of the 
fairness tolerance $\epsilon$ under the \texttt{min\_changes} objective. Our results demonstrate that satisfying fairness constraints 
via data bias mitigation does not meaningfully degrade predictive accuracy. The COMPAS dataset %
exhibits strong stability, with accuracy fluctuating  by only 1 percentage point (0.68--0.69). Notably, this stability holds regardless of the classifier used, with gradient boosting consistently achieving the highest accuracy across all datasets.

\begin{figure}[htbp]
    \centering
    \begin{subfigure}[b]{0.32\textwidth}
        \centering
        \includegraphics[width=\textwidth]{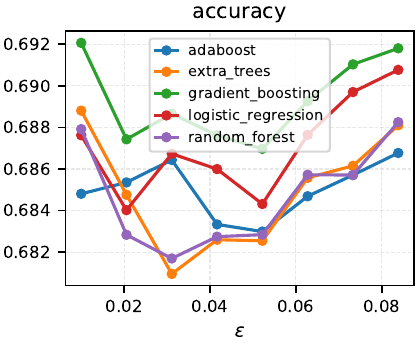}
        \caption{Fairness - \texttt{min\_changes} objective} %
        \label{fig:ml_eval_vary_epsilon_min_changes_compas}
    \end{subfigure}
    \hfill
    \begin{subfigure}[b]{0.32\textwidth}
        \centering
        \includegraphics[width=\textwidth]{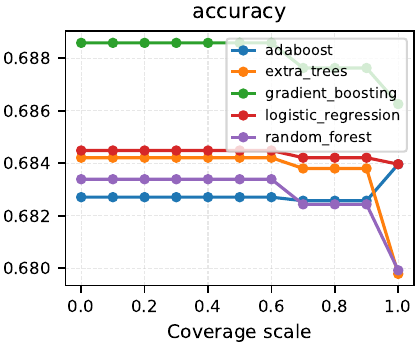}
        \caption{Coverage - \texttt{min\_changes} objective}
        \label{fig:ml_eval_vary_coverage_min_changes_compas}
    \end{subfigure}
    \hfill
    \begin{subfigure}[b]{0.32\textwidth}
        \centering
        \includegraphics[width=\textwidth]{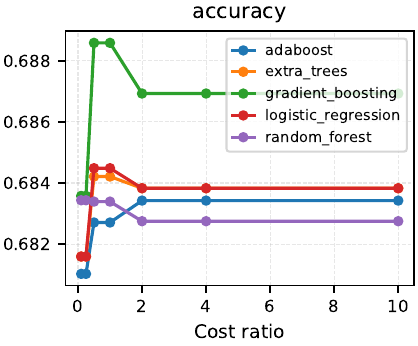}
        \caption{Cost - \texttt{min\_cost} objective}
        \label{fig:ml_eval_vary_eps_min-size_default}
    \end{subfigure}
    \caption{Experimental results for COMPAS. \answer{For (a) $m_{\textbf{s},y}=1$, (b) $\epsilon=0.05$, for (c) both. For (a) \& (b) $c_a=c_d=1$.} } %
    \Description{}
    \label{fig:ml_eval_min-changes_compas}
\end{figure}

\begin{figure}[htbp]
    \centering
    \begin{subfigure}[b]{0.32\textwidth}
        \centering
        \includegraphics[width=\textwidth]{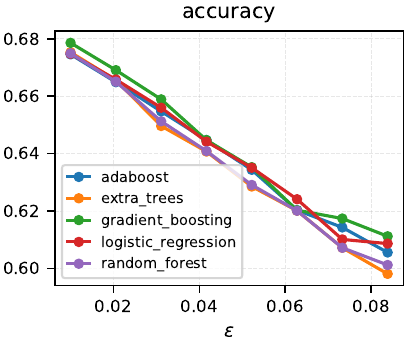}
        \caption{Fairness - \texttt{min\_size} objective}
        \label{fig:ml_eval_vary_epsilon_min_size_compas}
    \end{subfigure}
    \hfill
    \begin{subfigure}[b]{0.32\textwidth}
        \centering
        \includegraphics[width=\textwidth]{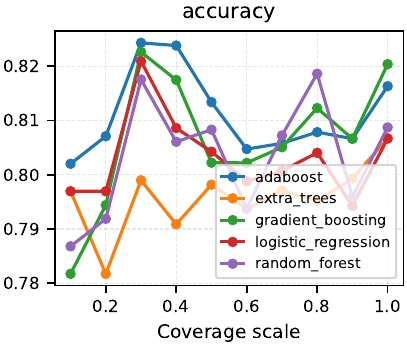}
        \caption{Coverage - \texttt{min\_size} objective}
        \label{fig:ml_eval_vary_coverage_min_size_default}
    \end{subfigure}
    \hfill
    \begin{subfigure}[b]{0.32\textwidth}
        \centering
        \includegraphics[width=\textwidth]{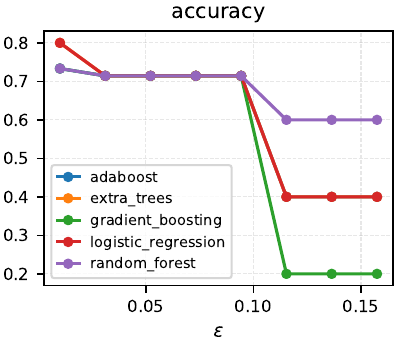}
        \caption{Fairness - \texttt{min\_size} objective}
        \label{fig:ml_eval_vary_eps_min_size_default}
    \end{subfigure}
    \caption{%
    Results for COMPAS (Left) and Default (Center and Right). \answerEqyy{For (a) \& (c) $m_{\textbf{s},y}=1$, for (b) $\epsilon=0.05$. For all $c_a=c_d=1$.}}
    \Description{}
    \label{fig:ml_eval_min-size_compas}
\end{figure}

A substantial body of work has documented trade-offs between fairness and predictive performance~\cite{corbett2023measure,menon2018cost} (a.k.a. fairness-accuracy trade-off), leading to concerns that bias mitigation may be impractical in performance-sensitive applications. Our results suggest that, at least for pre-processing approaches, satisfying fairness constraints need not incur a significant performance cost. When the optimization objective balances additions and deletions (\texttt{min\_changes}), the mitigated datasets maintain sufficient information for reliable downstream prediction.

\textbf{Optimization Objective Choice Affects Performance:} The \texttt{min\_size} objective reveals that the choice of optimization objective can have a larger impact on predictive performance than the fairness tolerance parameter itself. Figure~\ref{fig:ml_eval_vary_epsilon_min_size_compas} shows that under this objective, accuracy \emph{decreases} as $\epsilon$ \emph{increases}. For the COMPAS dataset, accuracy drops from approximately 0.68 to 0.60 across the $\epsilon$ range. The Default dataset demonstrates the most extreme case (Figure~\ref{fig:ml_eval_vary_eps_min_size_default}): at higher $\epsilon$ values, sample sizes collapse to as few as 24 observations, rendering reliable evaluation impossible and causing accuracy to drop significantly.
This behavior is a consequence of the optimization objective, not the fairness constraints. When fairness constraints are relaxed (higher $\epsilon$), the \texttt{min\_size} objective aggressively reduces the dataset to minimize storage, removing samples that, while not necessary for satisfying fairness constraints, are informative for prediction. Stricter fairness constraints inadvertently act as implicit coverage constraints, forcing the optimizer to retain a more representative sample. The \texttt{min\_size} objective would exhibit similar behavior under any constraint that permits deletion.
This finding provides empirical motivation for the coverage constraints we introduce in Section~\ref{sec:UB_coverage_algorithm}. Rather than relying on fairness constraints to implicitly preserve data representation, we propose explicit coverage guarantees $[\mathbf{s}y] + \Delta[\mathbf{s}y] \geq m_{\mathbf{s},y}$ that ensure sufficient representation of all group-label combinations regardless of the fairness tolerance. This decouples two distinct concerns (fairness and representation), allowing practitioners to specify minimum data requirements independently of their chosen $\epsilon$.

Previously, we identified that the \texttt{min\_size} objective, when unconstrained, can aggressively reduce datasets in ways that harm predictive performance. We now evaluate whether explicit coverage constraints, a central contribution of this work, can remedy this behavior while preserving the objective's utility for minimizing storage and computational costs. Figures~\ref{fig:ml_eval_vary_coverage_min_changes_compas} and \ref{fig:ml_eval_vary_coverage_min_size_default} present accuracy as a function of coverage scale for both the \texttt{min\_size} and \texttt{min\_changes} optimization objectives, respectively, across all 
three datasets. The main findings are:

\textbf{Coverage Constraints Restore Predictive Performance}: Under the \texttt{min\_size} objective, accuracy improves monotonically or remains stable across all datasets, as the degree to which the coverage constraints are satisfied increases. The Default dataset provides the most striking illustration of why coverage constraints are necessary. 
As observed with fairness tolerance, at coverage scale 0, the \texttt{min\_size} objective reduces the dataset to too few (merely 8) observations containing only a single label, rendering machine learning evaluation impossible.\footnote{This point is omitted from Figure~\ref{fig:ml_eval_vary_coverage_min_size_default} for this reason.} With even modest coverage requirements (Figure~\ref{fig:ml_eval_vary_coverage_min_size_default}, scale $\geq 0.1$), the dataset retains sufficient samples for meaningful prediction, achieving accuracy around 0.79--0.82. These results demonstrate that coverage constraints effectively decouple data retention from fairness requirements. Without coverage constraints, the \texttt{min\_size} objective pursues aggressive reduction that sacrifices predictive utility. With appropriate coverage guarantees, the same objective produces datasets that are both fair and useful for downstream learning tasks.

\textbf{The Minimizing Changes Objective is Naturally Robust.} In contrast, the \texttt{min\_changes} objective exhibits remarkable stability across coverage scales. For all three datasets, accuracy remains essentially constant from coverage scale 0 through 0.8, with only minor variation at scales 0.9--1.0. Sample sizes similarly remain fixed: 16,339 for Adult, 36,572 for COMPAS, and 9,845 for Default across most coverage values. This stability occurs because \texttt{min\_changes} inherently balances additions and deletions to minimize total modifications. Unlike \texttt{min\_size}, which has an incentive to delete as much as possible, \texttt{min\_changes} treats deletions (like additions) as costly operations to be avoided. The objective thus naturally preserves data representation without requiring explicit coverage constraints.
The cost ratio parameter $c_d/c_a$ controls the relative penalty for additions versus deletions. Low ratios favor deletion-heavy strategies; high ratios favor additions. Figure~\ref{fig:ml_eval_vary_cost_ratio} presents accuracy as a function of cost ratio across all three datasets. Accuracy remains stable across the full range of cost ratios tested (0.1 to 10.0). For COMPAS, accuracy varies by less than 1\%. This stability has practical implications. Practitioners may have operational reasons to prefer additions over deletions (\eg data augmentation is feasible) or vice versa (\eg storage constraints favor smaller datasets or new data is costly). Our results suggest that this choice can be made based on operational considerations without sacrificing predictive performance. %

\begin{figure}[htbp]
    \centering
    \begin{subfigure}[b]{0.32\textwidth}
        \centering
        \includegraphics[width=\textwidth]{figures/ml_eval/vary_cost_ratio/compas_Sex_Code_Text_race_binary_ScoreText_ml_eval_vary_cost_ratio_min_cost_accuracy_compact.pdf}
    \end{subfigure}
    \hfill
    \begin{subfigure}[b]{0.32\textwidth}
        \centering
        \includegraphics[width=\textwidth]{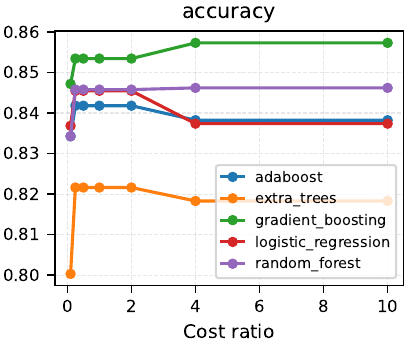}
    \end{subfigure}
    \hfill
    \begin{subfigure}[b]{0.32\textwidth}
        \centering
        \includegraphics[width=\textwidth]{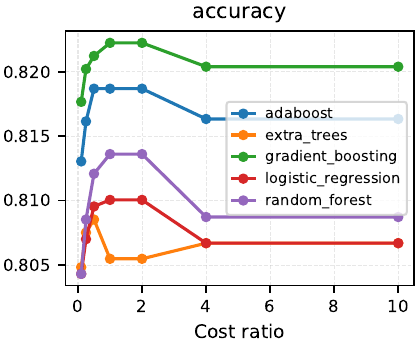}
    \end{subfigure}
    \caption{Accuracy vs.\ cost ratio %
    for the \texttt{min\_cost} objective \answer{($\epsilon=0.05$, $m_{\textbf{s},y}=1$)}. Left: COMPAS, Center: Adult, Right: Default.}
    \Description{}
    \label{fig:ml_eval_vary_cost_ratio}
\end{figure}
\section{Conclusions and future work}\label{sec:conclusions}

We designed bias mitigation algorithms that incorporate {\em coverage constraints}, enforcing sufficient representation across groups. Contrary to prior work, we permit both adding and deleting data, potentially with different costs, which is of practical interest. %
Since achieving exactly zero bias for all groups  may require large amounts of data, our closed-form algorithm trades small approximation errors in bias for less data usage. Beyond this, we formulate bias mitigation as an ILP that finds globally optimal solutions for any given fairness tolerance. %
This allows us to characterize the price of fairness, the minimum data modification cost as a function of fairness tolerance. We show that under minimal or no coverage constraints, fairness requirements can eliminate valuable information in data, which can lead to degradation of ML performance. We evaluate our mitigation algorithms on widely used real datasets and show that they preserve predictive accuracy across multiple classifiers. We show that coverage constraints are essential for this purpose.

\answer{
\textbf{Limitations.} Our framework assumes that additional data with the required group-label profile can be obtained, whether from data lakes, open datasets, or data brokers. In practice, data for underrepresented groups may be scarce, expensive, or unavailable~\cite{chen2018my}, and targeted data collection raises ethical concerns about over-surveilling specific populations~\cite{barocas2023fairness}. Our cost model abstracts these challenges through addition costs~$c_a$ and budget constraints~$B$, and upper bound constraints (Appendix~\ref{app:upper_bound_constraints}) can cap collection per group to limit over-collection. However, the framework does not account for distribution shift between external sources and the original dataset, nor does it address whether acquired data faithfully represents the target population. These are inherent challenges in data-centric fairness approaches~\cite{chen2018my,barocas2023fairness}, not specific to our method. Practitioners should assess data availability and collection feasibility before applying our framework in practice.} \answerHgb{Finally, while our framework operates on any categorical sensitive attributes, the choice of which attributes to designate as sensitive, and what fairness thresholds to impose, involves normative judgments that our technical framework does not resolve. The distinction between statistical bias (a property of data distributions) and unlawful discrimination (a legal determination) is context-dependent, and the ``price of fairness'' should be understood as the cost of achieving a \emph{chosen} fairness standard, not as a claim that discrimination has a tolerable price.}
Importantly, our framework permits multiple (non-binary) sensitive attributes and so can model bias over intersectional groups as well as coverage constraints for these groups.
In future work, we aim to formalize a definition of intersectional bias that isolates the non-additive component of bias in the general setting. %
We also intend to address bias mitigation in the presence of incomplete data, where datasets contain null values. %

\begin{acks}
We acknowledge the support of the NSF (IIS-2325632) and of the Canada Excellence Research Chairs (CERC) program. Nous remercions le programme des Chaires d'excellence en recherche du Canada (CERC) de son soutien. We thank the anonymous reviewers for their constructive and thoughtful feedback.
\end{acks}
\bibliographystyle{ACM-Reference-Format}
\bibliography{bibliography}
\section*{Generative AI Usage Statement}

The authors used Claude (Anthropic, models Sonnet 4.5 and Opus 4.5) occasionally for grammar checking of passages, formatting of tables and style editing of the manuscript. 
All content was reviewed and verified by the authors.

\appendix
\section{Dataset Details}
\label{app:datasets}

This appendix provides detailed bias analysis for each dataset used in our experimental evaluation. For each dataset, we report the group-label distributions, the absolute difference $|f_{\mathbf{s},y} - f_y|$ used as a proxy in our ILP formulation (Section~\ref{sec:price_of_fairness}), and the corresponding Uniform Bias (UB) values.

\noindent
The group-label distributions for fully specified groups (no wildcards $*$) are shown in Tables \ref{tab:adult-counts}, \ref{tab:compas-counts} and \ref{tab:default-counts} for the Adult, COMPAS and Default datasets, respectively. Table~\ref{tab:compas_ini_values} includes the aggregate groups for the COMPAS dataset.

\noindent
In the bias analysis tables (Tables~\ref{tab:adult-bias},~\ref{tab:compas-bias} and~\ref{tab:default-bias}), $f_{\mathbf{s},y}$ denotes the proportion of label $y$ within group $\mathbf{s}$, $f_y$ denotes the overall proportion of label $y$ in the dataset, along with their absolute difference and the Uniform Bias (UB) for that group-label pair. We report biases for aggregate groups (including wildcards $*$) and fully specified groups. Within each table, groups attribute values are listed in lexicographic order, except for education levels which are ordered by decreasing education level (Graduate School, University, High School, Other).

\begin{table}[ht]
\centering
\begin{tabular}{llrrr}
\toprule
\textbf{Sex} & \textbf{Race} & \textbf{$\leq$50K} & \textbf{$>$50K} & \textbf{Total} \\
\midrule
Female & Non-White &  2,938 &   227 &  3,165 \\
Female & White     & 11,485 & 1,542 & 13,027 \\
Male   & Non-White &  3,062 &   853 &  3,915 \\
Male   & White     & 19,670 & 9,065 & 28,735 \\
\midrule
\multicolumn{2}{l}{\textbf{Total}} & 37,155 & 11,687 & 48,842 \\
\bottomrule
\end{tabular}
\caption{Group-label counts for the Adult dataset.}
\label{tab:adult-counts}
\end{table} %

\begin{table}[ht]
\centering
\begin{tabular}{llrrrr}
\toprule
\textbf{Sex} & \textbf{Race} & \textbf{Low} & \textbf{Medium} & \textbf{High} & \textbf{Total} \\
\midrule
Female & Caucasian     &  4,159 &   894 &   375 &  5,428 \\
Female & Non-Caucasian &  5,637 & 1,589 &   665 &  7,891 \\
Male   & Caucasian     & 12,202 & 2,862 & 1,273 & 16,337 \\
Male   & Non-Caucasian & 19,489 & 7,143 & 4,510 & 31,142 \\
\midrule
\multicolumn{2}{l}{\textbf{Total}} & 41,487 & 12,488 & 6,823 & 60,798 \\
\bottomrule
\end{tabular}
\caption{Group-label counts for the COMPAS dataset.}
\label{tab:compas-counts}
\end{table} %

\begin{table}[t]
    \centering
    \begin{tabular}{|c|c|c|c|c|}
        \hline
        {\textbf{Sex / Race}}  & {\textbf{Label}} & \textbf{Caucasian (c)} & \textbf{Non-Caucasian (o)} & \textbf{Total}\\
        \hline
        {} & \textbf{L} & 4159 & 5637 & 9796 \\
        {Female (\textbf{w})} & M & 894 & 1589 & 2483 \\
        {} & \textbf{H} & 375 & 665 & 1040 \\
        {} & \textbf{Tot.} & \textbf{wc}: 5428 & \textbf{wo}: 7891 & \textbf{w}: 13319 \\
        \hline
        {} & \textbf{L} & 12202 & 19489 & 31691 \\
        {Male (\textbf{m})} & \textbf{M} & 2862 & 7143 & 10005 \\
        {} & \textbf{H} & 1273 & 4510 & 5783 \\
        {} & \textbf{Tot.} & \textbf{mc}: 16337 & \textbf{mo}: 31142 & \textbf{m}: 47479\\
        \hline
        {} & \textbf{L} & 16361 & 25126 & 41487 \\
        {\textbf{Total}} & \textbf{M} & 3756 & 8732 & 12488 \\
        {} & \textbf{H} & 1648 & 5175 & 6823 \\
        {} & \textbf{Tot.} & \textbf{c}: 21765 & \textbf{o}: 39033 & \textbf{n}: 60798 \\
        \hline
    \end{tabular}
    \caption{Summary statistics (values) for the initial version of the COMPAS dataset with ternary label.}
    \label{tab:compas_ini_values}
\end{table}

\begin{table}[ht]
\centering
\begin{tabular}{llrrr}
\toprule
\textbf{Sex} & \textbf{Education} & \textbf{No Default} & \textbf{Default} & \textbf{Total} \\
\midrule
Female & Graduate School & 5,101 & 1,130 & 6,231 \\
Female & University      & 6,734 & 1,922 & 8,656 \\
Female & High School     & 2,235 &   692 & 2,927 \\
Female & Other           &   279 &    19 &   298 \\
Male   & Graduate School & 3,448 &   906 & 4,354 \\
Male   & University      & 3,966 & 1,408 & 5,374 \\
Male   & High School     & 1,445 &   545 & 1,990 \\
Male   & Other           &   156 &    14 &   170 \\
\midrule
\multicolumn{2}{l}{\textbf{Total}} & 23,364 & 6,636 & 30,000 \\
\bottomrule
\end{tabular}
\caption{Group-label counts for the Default dataset.}
\label{tab:default-counts}
\end{table} %
\begin{table}[ht]
\centering
\begin{tabular}{llrrrr}
\toprule
\textbf{Group} & \textbf{Label} & $f_{\mathbf{s},y}$ & $f_y$ & $|f_{\mathbf{s},y} - f_y|$ & \textbf{Uniform Bias} \\
\midrule
\multicolumn{6}{l}{\textit{Aggregate groups (sex)}} \\
Female, $*$    & $\leq$50K & 0.891 & 0.761 & 0.130 & $-$0.171 \\
Female, $*$    & $>$50K    & 0.109 & 0.239 & 0.130 &  0.543 \\
Male, $*$      & $\leq$50K & 0.696 & 0.761 & 0.064 &  0.085 \\
Male, $*$      & $>$50K    & 0.304 & 0.239 & 0.064 & $-$0.269 \\
\midrule
\multicolumn{6}{l}{\textit{Aggregate groups (race)}} \\
$*$, Non-White & $\leq$50K & 0.847 & 0.761 & 0.087 & $-$0.114 \\
$*$, Non-White & $>$50K    & 0.153 & 0.239 & 0.087 &  0.362 \\
$*$, White     & $\leq$50K & 0.746 & 0.761 & 0.015 &  0.019 \\
$*$, White     & $>$50K    & 0.254 & 0.239 & 0.015 & $-$0.061 \\
\midrule
\multicolumn{6}{l}{\textit{Fully specified groups}} \\
Female, Non-White & $\leq$50K & 0.928 & 0.761 & 0.168 & $-$0.220 \\
Female, Non-White & $>$50K    & 0.072 & 0.239 & 0.168 &  0.700 \\
Female, White     & $\leq$50K & 0.882 & 0.761 & 0.121 & $-$0.159 \\
Female, White     & $>$50K    & 0.118 & 0.239 & 0.121 &  0.505 \\
Male, Non-White & $\leq$50K & 0.782 & 0.761 & 0.021 & $-$0.028 \\
Male, Non-White & $>$50K    & 0.218 & 0.239 & 0.021 &  0.089 \\
Male, White     & $\leq$50K & 0.685 & 0.761 & 0.076 &  0.100 \\
Male, White     & $>$50K    & 0.315 & 0.239 & 0.076 & $-$0.318 \\
\bottomrule
\end{tabular}
\caption{Bias analysis for the Adult dataset.}
\label{tab:adult-bias}
\end{table}

\begin{table}[ht]
\centering
\begin{tabular}{llrrrr}
\toprule
\textbf{Group} & \textbf{Label} & $f_{\mathbf{s},y}$ & $f_y$ & $|f_{\mathbf{s},y} - f_y|$ & \textbf{Uniform Bias} \\
\midrule
\multicolumn{6}{l}{\textit{Aggregate groups (sex)}} \\
Female, $*$ & Low    & 0.735 & 0.682 & 0.053 & $-$0.078 \\
Female, $*$ & Medium & 0.186 & 0.205 & 0.019 &  0.092 \\
Female, $*$ & High   & 0.078 & 0.112 & 0.034 &  0.304 \\
Male, $*$   & Low    & 0.667 & 0.682 & 0.015 &  0.022 \\
Male, $*$   & Medium & 0.211 & 0.205 & 0.005 & $-$0.026 \\
Male, $*$   & High   & 0.122 & 0.112 & 0.010 & $-$0.085 \\
\midrule
\multicolumn{6}{l}{\textit{Aggregate groups (race)}} \\
$*$, Caucasian     & Low    & 0.752 & 0.682 & 0.069 & $-$0.102 \\
$*$, Caucasian     & Medium & 0.173 & 0.205 & 0.033 &  0.160 \\
$*$, Caucasian     & High   & 0.076 & 0.112 & 0.037 &  0.325 \\
$*$, Non-Caucasian & Low    & 0.644 & 0.682 & 0.039 &  0.057 \\
$*$, Non-Caucasian & Medium & 0.224 & 0.205 & 0.018 & $-$0.089 \\
$*$, Non-Caucasian & High   & 0.133 & 0.112 & 0.020 & $-$0.181 \\
\midrule
\multicolumn{6}{l}{\textit{Fully specified groups}} \\
Female, Caucasian     & Low    & 0.766 & 0.682 & 0.084 & $-$0.123 \\
Female, Caucasian     & Medium & 0.165 & 0.205 & 0.041 &  0.198 \\
Female, Caucasian     & High   & 0.069 & 0.112 & 0.043 &  0.384 \\
Female, Non-Caucasian & Low    & 0.714 & 0.682 & 0.032 & $-$0.047 \\
Female, Non-Caucasian & Medium & 0.201 & 0.205 & 0.004 &  0.020 \\
Female, Non-Caucasian & High   & 0.084 & 0.112 & 0.028 &  0.249 \\
Male, Caucasian     & Low    & 0.747 & 0.682 & 0.065 & $-$0.095 \\
Male, Caucasian     & Medium & 0.175 & 0.205 & 0.030 &  0.147 \\
Male, Caucasian     & High   & 0.078 & 0.112 & 0.034 &  0.306 \\
Male, Non-Caucasian & Low    & 0.626 & 0.682 & 0.057 &  0.083 \\
Male, Non-Caucasian & Medium & 0.229 & 0.205 & 0.024 & $-$0.117 \\
Male, Non-Caucasian & High   & 0.145 & 0.112 & 0.033 & $-$0.290 \\
\bottomrule
\end{tabular}
\caption{Bias analysis for the COMPAS dataset.}
\label{tab:compas-bias}
\end{table}

\begin{table}[ht]
\centering
\begin{tabular}{llrrrr}
\toprule
\textbf{Group} & \textbf{Label} & $f_{\mathbf{s},y}$ & $f_y$ & $|f_{\mathbf{s},y} - f_y|$ & \textbf{Uniform Bias} \\
\midrule
\multicolumn{6}{l}{\textit{Aggregate groups (sex)}} \\
Female, $*$ & No Default & 0.792 & 0.779 & 0.013 & $-$0.017 \\
Female, $*$ & Default    & 0.208 & 0.221 & 0.013 &  0.061 \\
Male, $*$   & No Default & 0.758 & 0.779 & 0.020 &  0.026 \\
Male, $*$   & Default    & 0.242 & 0.221 & 0.020 & $-$0.093 \\
\midrule
\multicolumn{6}{l}{\textit{Aggregate groups (education)}} \\
$*$, Graduate School & No Default & 0.808 & 0.779 & 0.029 & $-$0.037 \\
$*$, Graduate School & Default    & 0.192 & 0.221 & 0.029 &  0.130 \\
$*$, University      & No Default & 0.763 & 0.779 & 0.016 &  0.021 \\
$*$, University      & Default    & 0.237 & 0.221 & 0.016 & $-$0.073 \\
$*$, High School     & No Default & 0.748 & 0.779 & 0.030 &  0.039 \\
$*$, High School     & Default    & 0.252 & 0.221 & 0.030 & $-$0.137 \\
$*$, Other           & No Default & 0.929 & 0.779 & 0.151 & $-$0.193 \\
$*$, Other           & Default    & 0.071 & 0.221 & 0.151 &  0.681 \\
\midrule
\multicolumn{6}{l}{\textit{Fully specified groups}} \\
Female, Graduate School & No Default & 0.819 & 0.779 & 0.040 & $-$0.051 \\
Female, Graduate School & Default    & 0.181 & 0.221 & 0.040 &  0.180 \\
Female, University      & No Default & 0.778 & 0.779 & 0.001 &  0.001 \\
Female, University      & Default    & 0.222 & 0.221 & 0.001 & $-$0.004 \\
Female, High School     & No Default & 0.764 & 0.779 & 0.015 &  0.020 \\
Female, High School     & Default    & 0.236 & 0.221 & 0.015 & $-$0.069 \\
Female, Other           & No Default & 0.936 & 0.779 & 0.157 & $-$0.202 \\
Female, Other           & Default    & 0.064 & 0.221 & 0.157 &  0.712 \\
Male, Graduate School & No Default & 0.792 & 0.779 & 0.013 & $-$0.017 \\
Male, Graduate School & Default    & 0.208 & 0.221 & 0.013 &  0.059 \\
Male, University      & No Default & 0.738 & 0.779 & 0.041 &  0.052 \\
Male, University      & Default    & 0.262 & 0.221 & 0.041 & $-$0.184 \\
Male, High School     & No Default & 0.726 & 0.779 & 0.053 &  0.068 \\
Male, High School     & Default    & 0.274 & 0.221 & 0.053 & $-$0.238 \\
Male, Other           & No Default & 0.918 & 0.779 & 0.139 & $-$0.178 \\
Male, Other           & Default    & 0.082 & 0.221 & 0.139 &  0.628 \\
\bottomrule
\end{tabular}
\caption{Bias analysis for the Default dataset.}
\label{tab:default-bias}
\end{table}

\section{Coverage-Constrained Mitigation}\label{app:mitigation_coverage}

This appendix provides supplementary material for Section~\ref{sec:UB_coverage_algorithm}. We present detailed experimental results demonstrating our coverage-constrained mitigation algorithm on the COMPAS dataset, along with additional technical details omitted from the main text for brevity.

\subsection{Relaxing fairness conditions}\label{app:nonzero_bias_condition} %

The unbiased condition can be generalized~\cite{scarone2025principled} to allow a data scientist to specify a number $K_{\textbf{s},y} > 0$ for each group \textbf{s} and label value $y$, which is the desired ratio of the success rate of the protected group to the success rate of the population:
\[
    f_{\textbf{s},y} / f_y = K_{\textbf{s},y}
\]
To recover the setting used in Section~\ref{sec:UB_coverage_algorithm}, where zero bias is the goal, one can set $K_{\textbf{s},y}=1$. 

\noindent
With this general unbiased condition, 
the new objective becomes
\[
    f^{\text{new}}_{\textbf{s},y} = \frac{[\textbf{s}y]+\Delta[\textbf{s}y]}{\textbf{s}+\Delta\textbf{s}}=f_y\cdot K_{\textbf{s},y} \Leftrightarrow \frac{[\textbf{s}y]+\Delta[\textbf{s}y]}{f_y\cdot K_{\textbf{s},y}} = \textbf{s}+\Delta\textbf{s}
\]
As was done by Scarone et al. \cite{scarone2025principled}, we can equate the equations corresponding to $[\textbf{s}y]$ and $[\textbf{s}y_i]$, obtaining 
\[
    \frac{[\textbf{s}y]+\Delta[\textbf{s}y]}{f_y\cdot K_{\textbf{s},y}} = \frac{[\textbf{s}y_i]+\Delta[\textbf{s}y_i]}{f_{y_i}\cdot K_{\textbf{s},y_i}}
\]
for each group-label. Solving for $\Delta[\textbf{s}y]$ we get the general version of Equation~\ref{eq:UB_mitigation_linear_system}:%
\[
    \Delta[\textbf{s}y] = -[\textbf{s}y] + \frac{y}{y_i}\frac{K_{\textbf{s},y}}{K_{\textbf{s},y_i}}\Big([\textbf{s}y_i] + \Delta[\textbf{s}y_i]\Big) 
\]

\subsection{Full solutions for COMPAS dataset}
\label{app:compas_coverage}

We apply the closed-form mitigation algorithm to the COMPAS dataset with two sensitive attributes (Sex and Race) and three risk score labels (Low, Medium, High). 

\subsubsection{Exact Solution}

Table~\ref{tab:compas_closed_form_exact} shows the exact solution with $m_{\textbf{s},y} = 1$, which uses $\Delta[\textbf{s}y] = -[\textbf{s}y] + k \cdot y$ where $k = \lceil m_{\textbf{s},y} / y \rceil$. This achieves exactly zero bias for all group-label pairs. Note that the exact solution without  coverage constraints %
results in $k = 1$ for all labels, yielding $[\textbf{s}y]^{\text{new}} = [\textbf{s}y] + \Delta[\textbf{s}y]= y$. This is also the smallest valid solution: since $k^* \geq \lceil m_{\textbf{s},y}/y \rceil = 1$ for $m_{\textbf{s},y} = 1$, any $k < 1$ yields an invalid solution. For instance, $k=0$ gives $[\textbf{s}y]^{\text{new}} = 0$, violating the coverage constraint, and $k=-1$ gives $|\Delta[\textbf{s}y]| = |-[\textbf{s}y]-y| > [\textbf{s}y]$, requiring more deletions than available tuples from $(\textbf{s},y)$ in the dataset. Thus, $k=1$ represents the minimum modification required to achieve zero bias under the minimal coverage constraint $m_{\textbf{s},y} = 1$.

\begin{table}
\begin{tabular}{llllllll}
\toprule
Group $\textbf{s}$ & Label $y$ & $[\textbf{s}y]$ & $\Delta[sy]$ & $[\textbf{s}y]^{\text{new}}$ & $f_{\textbf{s},y}^{\text{new}}$ & $f_y$ & Uniform Bias \\
\midrule
(Female, Caucasian) & Low & 4159 & 37328 & 41487 & 0.6824 & 0.6824 & 0.0000 \\
(Female, Caucasian) & Medium & 894 & 11594 & 12488 & 0.2054 & 0.2054 & 0.0000 \\
(Female, Caucasian) & High & 375 & 6448 & 6823 & 0.1122 & 0.1122 & 0.0000 \\
(Female, Non-Caucasian) & Low & 5637 & 35850 & 41487 & 0.6824 & 0.6824 & 0.0000 \\
(Female, Non-Caucasian) & Medium & 1589 & 10899 & 12488 & 0.2054 & 0.2054 & 0.0000 \\
(Female, Non-Caucasian) & High & 665 & 6158 & 6823 & 0.1122 & 0.1122 & 0.0000 \\
(Male, Caucasian) & Low & 12202 & 29285 & 41487 & 0.6824 & 0.6824 & 0.0000 \\
(Male, Caucasian) & Medium & 2862 & 9626 & 12488 & 0.2054 & 0.2054 & 0.0000 \\
(Male, Caucasian) & High & 1273 & 5550 & 6823 & 0.1122 & 0.1122 & 0.0000 \\
(Male, Non-Caucasian) & Low & 19489 & 21998 & 41487 & 0.6824 & 0.6824 & 0.0000 \\
(Male, Non-Caucasian) & Medium & 7143 & 5345 & 12488 & 0.2054 & 0.2054 & 0.0000 \\
(Male, Non-Caucasian) & High & 4510 & 2313 & 6823 & 0.1122 & 0.1122 & 0.0000 \\
\bottomrule
\end{tabular}
\caption{Closed-form exact solution for COMPAS}
\label{tab:compas_closed_form_exact}
\end{table} %

Table~\ref{tab:compas_closed_form_approximate} shows the results using the approximate solution with coverage constraints $m_{\textbf{s},y} = 1{,}000$ and reference label selection $i = \arg\max_j [\textbf{s}y_j]/[y_j]$ (which determines the free variable for each group $\textbf{s}$).

The algorithm achieves near-zero bias across all group-label pairs, with an average residual bias of $0.00014$. The small non-zero values arise from the ceiling operation in the computation of $\Delta[\textbf{s}y_i]$ (which ensures the coverage constraint is satisfied), the floor operation when upper bound constraints are present (Section~\ref{app:upper_bound_constraints}), and the rounding operation in the computation of $\Delta[\textbf{s}y]$.

\begin{table}
\begin{tabular}{llllllll}
\toprule
Group $\textbf{s}$ & Label $y$ & $[\textbf{s}y]$ & $\Delta[\textbf{s}y]$ & $[\textbf{s}y]^{\text{new}}$ & $f_{\textbf{s},y}^{\text{new}}$ & $f_y$ & Uniform Bias \\
\midrule
(Female, Caucasian) & Low & 4159 & 1922 & 6081 & 0.6823 & 0.6823 & 0.0000 \\
(Female, Caucasian) & High & 375 & 626 & 1001 & 0.1123 & 0.1123 & -0.0002 \\
(Female, Caucasian) & Medium & 894 & 937 & 1831 & 0.2054 & 0.2054 & 0.0000 \\
(Female, Non-Caucasian) & Low & 5637 & 444 & 6081 & 0.6823 & 0.6823 & 0.0000 \\
(Female, Non-Caucasian) & High & 665 & 336 & 1001 & 0.1123 & 0.1123 & -0.0002 \\
(Female, Non-Caucasian) & Medium & 1589 & 242 & 1831 & 0.2054 & 0.2054 & 0.0000 \\
(Male, Caucasian) & Low & 12202 & -6121 & 6081 & 0.6823 & 0.6823 & 0.0000 \\
(Male, Caucasian) & High & 1273 & -272 & 1001 & 0.1123 & 0.1123 & -0.0002 \\
(Male, Caucasian) & Medium & 2862 & -1031 & 1831 & 0.2054 & 0.2054 & 0.0000 \\
(Male, Non-Caucasian) & Low & 19489 & -13408 & 6081 & 0.6823 & 0.6823 & -0.0001 \\
(Male, Non-Caucasian) & High & 4510 & -3510 & 1000 & 0.1122 & 0.1123 & 0.0007 \\
(Male, Non-Caucasian) & Medium & 7143 & -5312 & 1831 & 0.2055 & 0.2054 & -0.0001 \\
\bottomrule
\end{tabular}
\caption{Closed-form approximate solution for COMPAS}
\label{tab:compas_closed_form_approximate}
\end{table} %

Note that $f_{\textbf{s},y}^{\text{new}} \approx f_y$ for all rows, confirming that the modified data approximately satisfies the unbiased condition $f_{\textbf{s},y} = f_y$, with small deviations due to rounding.

\subsection{Error bounds for approximate solution}

Theorem~\ref{thm:approx_errors_bounds}, used to establish the error bounds for the approximate bias mitigation algorithm, is a consequence of Theorem~\ref{thm:approx_errors_bounds_general_deltas}.

\begin{theorem}[Error bounds for approximate bias mitigation - General Form]\label{thm:approx_errors_bounds_general_deltas}
    Fix $k \in \mathbb{Z}^+$. Let $a, b \in \mathbb{Z}^+$ and $c, d \in \mathbb{Z}$ with $a \leq b$, $c \leq d$, $a + c \geq 0$, $b + d > 0$, and $a + c + 1 \leq b + d + k$. Define $E = \frac{a+c+1}{b+d+k} - \frac{a+c}{b+d}$. Then: $\frac{1-k}{b+d+k} \;\leq\; E \;\leq\; \frac{1}{b+d+k}$. Both bounds are tight.
\end{theorem}
\begin{proof}
    Let $n = a + c$ and $m = b + d$. The constraints $a \leq b$ and $c \leq d$ imply $n \leq m$. The constraint $a + c \geq 0$ gives $n \geq 0$.
    Before computing the bounds, we algebraically simplify $E$,
    \[
        E = \frac{n+1}{m+k} - \frac{n}{m} = \frac{(n+1)m - n(m+k)}{m(m+k)} = \frac{m - nk}{m(m+k)}
    \]

    Now, for the upper bound, since $n \geq 0$:
    \[
        m - nk \leq m
    \]
    with equality if and only if $n = 0$. Dividing by $m(m+k) > 0$, we obtain the upper bound:
    \[
        \frac{m - nk}{m(m+k)} \leq \frac{1}{m+k}
    \]

    For the lower bound, since $n \leq m$:
    \[
        m - nk \geq m(1-k)
    \]
    with equality if and only if $n = m$. Again, dividing by $m(m+k) > 0$, we obtain the lower bound:
    \[
        \frac{m - nk}{m(m+k)} \geq \frac{1-k}{m+k}
    \]

    Regarding tightness, note that
    \begin{itemize}
        \item The upper bound is achieved when $n = 0$, \eg $(a,b,c,d) = (1,1,-1,1)$.
        \item The lower bound is achieved when $n = m$, \eg $(a,b,c,d) = (1,1,1,1)$.
    \end{itemize}
\end{proof}
\begin{remark}
    Theorem~\ref{thm:approx_errors_bounds} is recovered by setting $a = [\textbf{s}y]$, $b = \textbf{s}$, $c = \Delta^{\text{ex}}_{\textbf{s},y}$, and $d = \Delta^{\text{ex}}_{\textbf{s}}$.
\end{remark}

\subsection{Upper bound constraints}\label{app:upper_bound_constraints}

Following the same approach as for lower bound constraints (Section~\ref{sec:coverage}) — for both exact and approximate solutions — the framework can also accommodate upper bound constraints of the form $[\textbf{s}y]+\Delta [\textbf{s}y] \leq M_{\textbf{s},y}$. For exact solutions, an upper bound constraint $[\textbf{s}y]+\Delta[\textbf{s}y] \leq M_{\textbf{s},y}$ translates to $k \leq \lfloor M_{\textbf{s},y}/y \rfloor$. Combining both bounds, the valid range of $k$ is $\lceil m_{\textbf{s},y}/y \rceil \leq k \leq \lfloor M_{\textbf{s},y}/y \rfloor$, and the most data-efficient exact solution satisfying both constraints is $k^* = \lceil m_{\textbf{s},y}/y \rceil$.

For approximate solutions, an upper bound constraint $[\textbf{s}y]+\Delta[\textbf{s}y] \leq M_{\textbf{s},y}$ translates to $\Delta[\textbf{s}y_i] \leq \frac{y_i}{y}M_{\textbf{s},y}-[\textbf{s}y_i]$ for each label $y$. Since we need to satisfy all labels simultaneously, the largest integral value of $\Delta[\textbf{s}y_i]$ that guarantees the upper bound constraint is:
\[
    \Delta[\textbf{s}y_i]\leq \min_{y}\Big\{\Big\lfloor\frac{y_i}{y}M_{\textbf{s},y}-[\textbf{s}y_i]\Big\rfloor\Big\}
\]
Combining both bounds (the lower bound from Section~\ref{sec:approx} and the upper bound above), we get all possible values of each free variable that produce valid solutions:
\[ 
    \max_{y}\Big\{\Big\lceil \frac{y_i}{y}m_{\textbf{s},y}-[\textbf{s}y_i]\Big\rceil\Big\}\leq\Delta[\textbf{s}y_i]\leq \min_{y}\Big\{\Big\lfloor\frac{y_i}{y}M_{\textbf{s},y}-[\textbf{s}y_i]\Big\rfloor\Big\}
\]

\section{External Source Distribution Estimation}\label{app:coverage_LB_estimation}

This appendix provides supplementary material for Section~\ref{sec:coverage_source_estimation}.

\begin{theorem}[Serfling's inequality]\label{thm:serfling}
    Let $Y_1, \dots, Y_n$ be random variables obtained by sampling without replacement from a finite population of $N$ values, each bounded in $[a, b]$. Let $X = \sum_{j=1}^n Y_j$ and let $\mu = E[X]$. Then, for any $t > 0$,
    \[
        \Pr[|X - \mu| \geq t] \leq 2\exp\left(-\frac{2t^2}{n(b-a)^2\left(1 - \frac{n-1}{N}\right)}\right).
    \]
\end{theorem}

\begin{remark}[Derivation from Bardenet \& Maillard]
    Theorem 2.4 in \cite{bardenet2015concentration} states that
    \[
        P\left(\max_{1 \leq k \leq n} \frac{\sum_{t=1}^k (X_t - \mu)}{N-k} \geq \frac{n\epsilon}{N-n}\right) \leq \exp\left(-\frac{2n\epsilon^2}{(1 - (n-1)/N)(b-a)^2}\right).
    \]
    Specializing to $k = n$, the left side simplifies to $P\left(\sum_{t=1}^n (X_t - \mu) \geq n\epsilon\right)$. Substituting $t = n\epsilon$ yields the one-sided bound $P(X - E[X] \geq t) \leq \exp\left(-\frac{2t^2}{n(1 - (n-1)/N)(b-a)^2}\right)$. Applying the same argument to $-X_t$ for the lower tail and taking a union bound gives the two-sided version stated in Theorem~\ref{thm:serfling}.
\end{remark}

\subsection{Asymptotic Analysis of the Sample Complexity}\label{app:sample-complexity}

Let $L = \ln(2c/\delta)$ for notational convenience. The required number of samples for source distribution estimation (Section~\ref{sec:coverage_source_estimation}) is given by:
\begin{equation}
    n = \frac{(N + 1)L}{L + 2\epsilon^2 N}.
\end{equation}
To analyze the asymptotic behavior as $N \to \infty$, we divide both numerator and denominator by $N$:
\begin{equation}
    n = \frac{(1 + 1/N)L}{L/N + 2\epsilon^2}.
\end{equation}
Taking the limit as $N \to \infty$:
\begin{equation}
    \lim_{N \to \infty} n = \frac{L}{2\epsilon^2} = \frac{\ln(2c/\delta)}{2\epsilon^2}.
\end{equation}
Therefore, the sample complexity is:
\begin{equation}
    n = O\left(\frac{\log(c/\delta)}{\epsilon^2}\right),
\end{equation}
which is independent of the dataset size $N$ for sufficiently large datasets.
\section{Linearization of the Price of Fairness ILP}

This appendix provides supplementary material for Section~\ref{sec:price_of_fairness}, presenting the linearization of the two non-linear components of the ILP and the complete integer linear program.

\subsection{Linearization of Constraints}\label{app:ILP_linearization}

The formulation presented in Section~\ref{sec:price_of_fairness} contains two sources of non-linearity: the approximate fairness constraint, which involves a ratio of decision variables, and the cost function, which is piecewise linear. We now show how both can be reformulated to obtain a pure integer linear program.

\subsubsection{Fairness Constraint}

The approximate fairness constraint bounds the deviation between each group's outcome distribution and the global target:
\begin{equation}
    \left| f^{\text{new}}_{\textbf{s},y} - f_y \right| \leq \epsilon
\end{equation}
where $f^{\text{new}}_{\textbf{s},y} = \frac{[\textbf{s}y] + \Delta[\textbf{s}y]}{\textbf{s} + \Delta\textbf{s}}$ is the proportion of label $y$ within group $\textbf{s}$ after mitigation, and $f_y$ is the global proportion of label $y$. Expanding the absolute value yields:
\begin{equation}
    -\epsilon \leq \frac{[\textbf{s}y] + \Delta[\textbf{s}y]}{\textbf{s} + \Delta\textbf{s}} - f_y \leq \epsilon
\end{equation}

Since $\textbf{s} + \Delta\textbf{s} > 0$ is guaranteed by the coverage constraints, we can cross-multiply without changing the inequality direction:
\begin{equation}
    -\epsilon(\textbf{s} + \Delta\textbf{s}) \leq [\textbf{s}y] + \Delta[\textbf{s}y] - f_y(\textbf{s} + \Delta\textbf{s}) \leq \epsilon(\textbf{s} + \Delta\textbf{s})
\end{equation}

Rearranging each inequality and substituting $\Delta\textbf{s} = \sum_{y'} \Delta[\textbf{s}y']$, we obtain two linear constraints:
\begin{align}
    \Delta[\textbf{s}y] - (f_y + \epsilon)\sum_{y'}\Delta[\textbf{s}y'] &\leq (f_y + \epsilon)\textbf{s} - [\textbf{s}y] \label{eq:fair-upper}\\
    \Delta[\textbf{s}y] - (f_y - \epsilon)\sum_{y'}\Delta[\textbf{s}y'] &\geq (f_y - \epsilon)\textbf{s} - [\textbf{s}y] \label{eq:fair-lower}
\end{align}

The right-hand sides are constants (computed from the original data), and the left-hand sides are linear in the decision variables $\Delta[\textbf{s}y]$. This linearization is exact, \ie no approximation is introduced.

\subsubsection{Piecewise Cost Function}

The cost function of modifying the dataset distinguishes between additions and deletions:
\begin{equation}
    \text{cost}(c_a, c_d, \Delta[\textbf{s}y]) = 
    \begin{cases} 
        c_a \cdot \Delta[\textbf{s}y] & \text{if } \Delta[\textbf{s}y] > 0 \\
        -c_d \cdot \Delta[\textbf{s}y] & \text{if } \Delta[\textbf{s}y] < 0 \\
        0 & \text{if } \Delta[\textbf{s}y] = 0
    \end{cases}
\end{equation}

To linearize this piecewise function, we decompose each decision variable into its positive and negative parts. For each group-label $(\textbf{s}, y)$, we introduce auxiliary variables $\Delta^+_{\textbf{s}y}, \Delta^-_{\textbf{s}y} \geq 0$ representing the number of additions and deletions respectively, with the constraint that the net change equals the difference between additions and deletions:
\begin{equation}
    \Delta[\textbf{s}y] = \Delta^+_{\textbf{s},y} - \Delta^-_{\textbf{s},y}
\end{equation}

The total cost then becomes a linear function of these auxiliary variables:
\begin{equation}
    \sum_{\textbf{s},y} \left( c_a \cdot \Delta^+_{\textbf{s},y} + c_d \cdot \Delta^-_{\textbf{s},y} \right)
\end{equation}

Note that we do not require an explicit constraint preventing both $\Delta^+_{\textbf{s},y}$ and $\Delta^-_{\textbf{s},y}$ from being positive simultaneously. Assuming $c_a, c_d > 0$ (which holds in all practical settings where both additions and deletions have a non-zero cost), any feasible solution with both variables positive for the same group-label can be improved by reducing both by $\min(\Delta^+_{\textbf{s},y}, \Delta^-_{\textbf{s},y})$, so the optimizer will naturally avoid such solutions (simultaneously adding and deleting from the same group-label is always suboptimal).

\subsection{Complete ILP Formulation}

Combining these linearizations, the complete integer linear program is as follows:
\begin{align}
    \min \quad & f(\{[\textbf{s}y]\},\{\Delta[\textbf{s}y]\})\\
    \text{s.t.} \quad 
    & \Delta[\textbf{s}y] - (f_y + \epsilon)\sum_{y'}\Delta[\textbf{s}y'] \leq (f_y + \epsilon)\textbf{s} - [\textbf{s}y] & \forall \textbf{s}, y \quad & \text{(fairness upper bound)}\\
    & \Delta[\textbf{s}y] - (f_y - \epsilon)\sum_{y'}\Delta[\textbf{s}y'] \geq (f_y - \epsilon)\textbf{s} - [\textbf{s}y] & \forall \textbf{s}, y \quad & \text{(fairness lower bound)}\\
    & [\textbf{s}y] + \Delta[\textbf{s}y] \geq m_{\textbf{s},y} & \forall \textbf{s}, y \quad & \text{(coverage)}\\
    & \Delta[\textbf{s}y] = \Delta^+_{\textbf{s},y} - \Delta^-_{\textbf{s},y} & \forall \textbf{s}, y \quad & \text{(decomposition)}\label{eq:decomposition}\\
    & \Delta^+_{\textbf{s},y}, \Delta^-_{\textbf{s},y} \geq 0 & \forall \textbf{s}, y \quad & \text{(non-negativity)}\\
    & \Delta^+_{\textbf{s},y}, \Delta^-_{\textbf{s},y} \in \mathbb{Z} & \forall \textbf{s}, y \quad & \text{(integrality)}\\
    &  \Delta\textbf{s} = \Delta[\textbf{s}y_1]+\dots+\Delta[\textbf{s}y_k] & \forall \textbf{s} \quad & \text{(definition of $\Delta\textbf{s}$)} \\
    & \textbf{s} = [\textbf{s}y_1]+\dots+[\textbf{s}y_k] & \forall \textbf{s} \quad & \text{(definition of $\textbf{s}$)} \\
    & \sum_{\textbf{s},y} \left( c_a \cdot \Delta^+_{\textbf{s},y} + c_d \cdot \Delta^-_{\textbf{s},y} \right)\leq B && \text{(cost and budget)}
\end{align}

The last two constraints make explicit the definitions of $\Delta\textbf{s}$ and $\textbf{s}$ used in the fairness constraints above, where $y_1, \dots, y_k$ denote all label values.

\begin{remark}
    The objective $f$ is a placeholder for any of the three objectives defined in Section~\ref{sec:price_of_fairness}. For \texttt{min\_size}, $\Delta[\textbf{s}y]$ can be substituted directly. For \texttt{min\_changes} and \texttt{min\_cost}, the objectives are already expressed in terms of the auxiliary variables $\Delta^+_{\textbf{s},y}$ and $\Delta^-_{\textbf{s},y}$ introduced by the decomposition constraint~\eqref{eq:decomposition}.
\end{remark}

\begin{remark}
    The ILP optimizes over all $\Delta[\textbf{s}y]$ simultaneously without fixing a reference label $y_i$ per group $\textbf{s}$ (as was done by Scarone et al.~\cite{scarone2025principled} and in Section~\ref{sec:approx}). Since any fixed-strategy solution is a feasible point in the ILP's solution space, the ILP's optimal cost is a lower bound on the cost of any fixed-strategy solution, and thus finds the globally optimal mitigation strategy.
\end{remark}
\section{The price of fairness: Additional ILP Experimental Results}\label{app:ILP_experiments}

This appendix contains the complete experimental results for Section~\ref{sec:price_of_fairness}, complementing the representative results shown in the main text with the remaining datasets for each experiment.

\paragraph{Fairness tolerance} 
The results are shown in Figures~\ref{fig:vary_eps_adult} (Adult) and~\ref{fig:vary_eps_default} (Default). 

\begin{figure}
    \centering
    \includegraphics[width=0.9\linewidth]
    {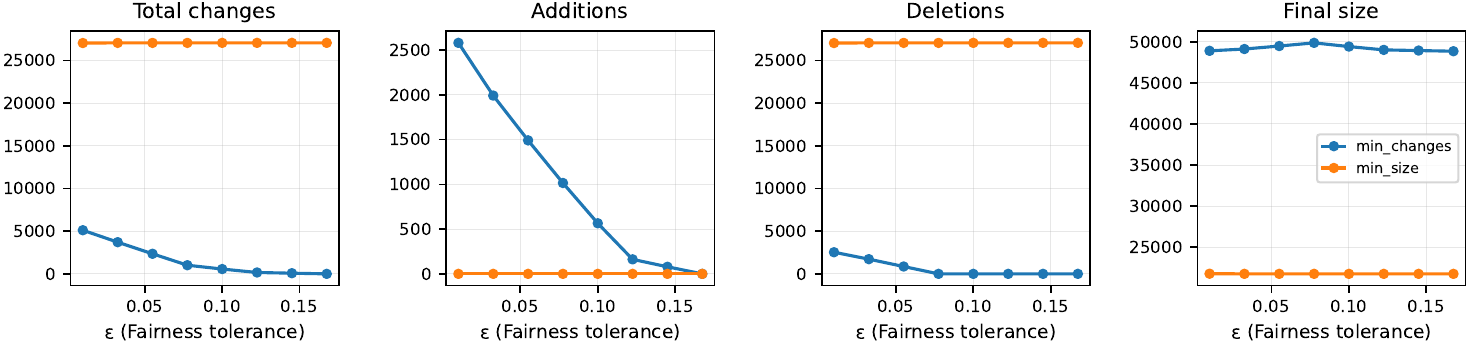}
    \caption{Effect of fairness tolerance ($\epsilon$) on mitigation solutions for the Adult dataset.}
    \Description{}
    \label{fig:vary_eps_adult}
\end{figure}

\begin{figure}
    \centering
    \includegraphics[width=0.9\linewidth]
    {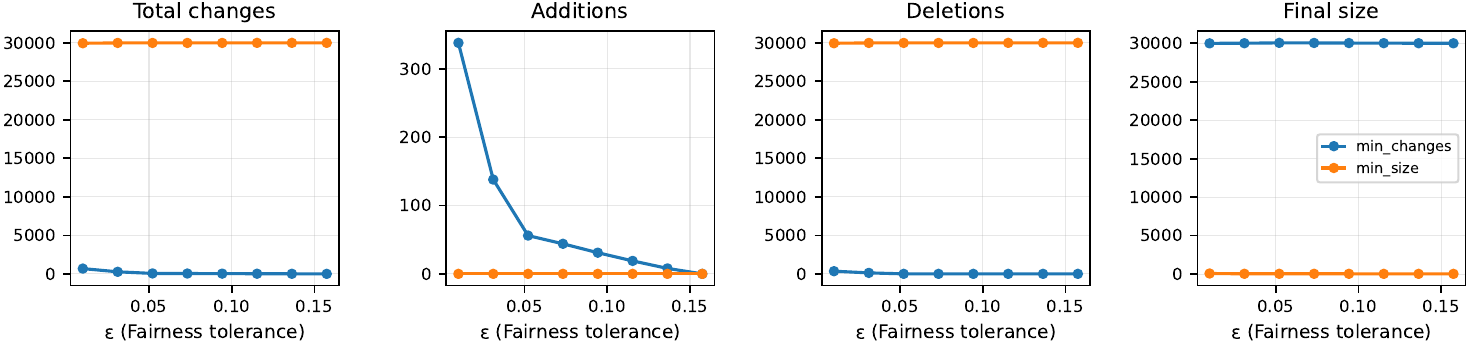}
    \caption{Effect of fairness tolerance ($\epsilon$) on mitigation solutions for the Default dataset.}
    \Description{}
    \label{fig:vary_eps_default}
\end{figure}

\paragraph{Coverage requirements}
The results are shown in Figures~\ref{fig:vary_coverage_compas} (COMPAS) and~\ref{fig:vary_coverage_default} (Default). 

\begin{figure}
    \centering
    \includegraphics[width=0.9\linewidth]{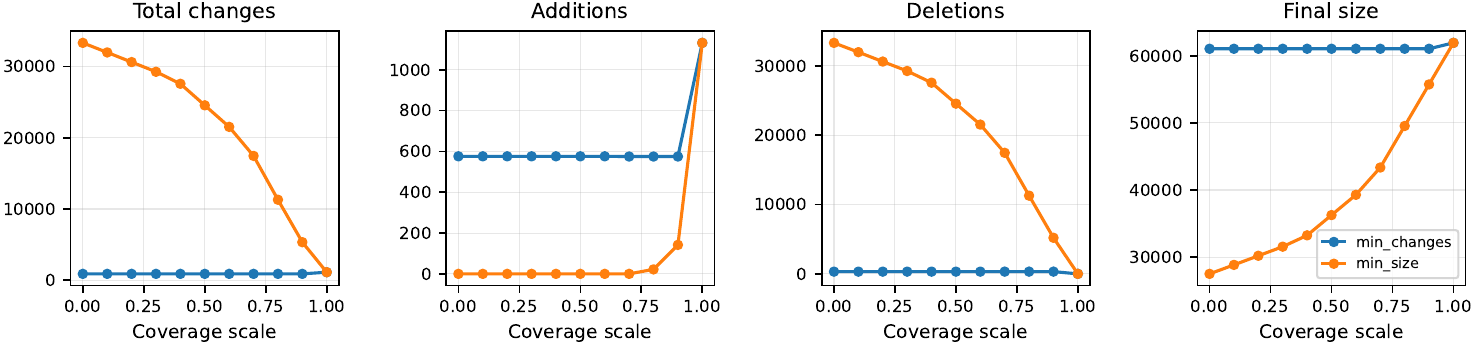}
    \caption{Effect of coverage scale on mitigation solutions for the COMPAS dataset.}
    \Description{}
    \label{fig:vary_coverage_compas}
\end{figure}

\begin{figure}
    \centering
    \includegraphics[width=0.9\linewidth]{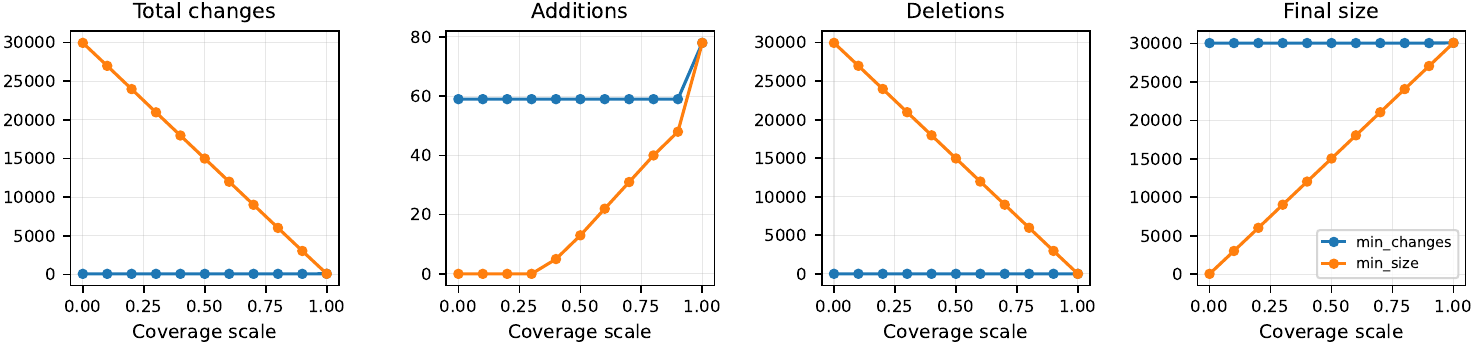}
    \caption{Effect of coverage scale on mitigation solutions for the Default dataset.}
    \Description{}
    \label{fig:vary_coverage_default}
\end{figure}

\paragraph{Cost asymmetry}
The results are shown in Figures~\ref{fig:vary_cost_adult} (Adult) and~\ref{fig:vary_cost_compas} (COMPAS).

\begin{figure}
    \centering
    \includegraphics[width=0.9\linewidth]{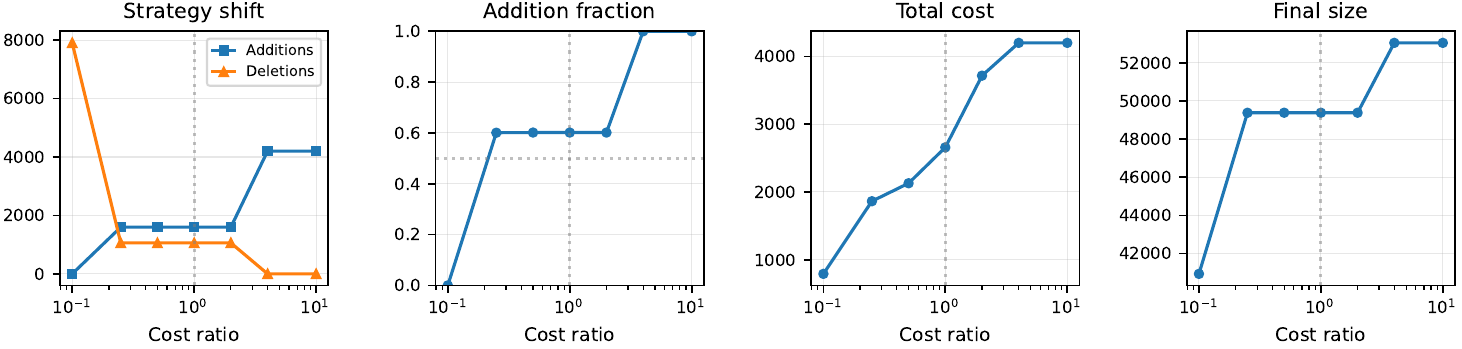}
    \caption{Effect of cost ratio on mitigation solutions for the Adult dataset.}
    \Description{}
    \label{fig:vary_cost_adult}
\end{figure}

\begin{figure}
    \centering
    \includegraphics[width=0.9\linewidth]{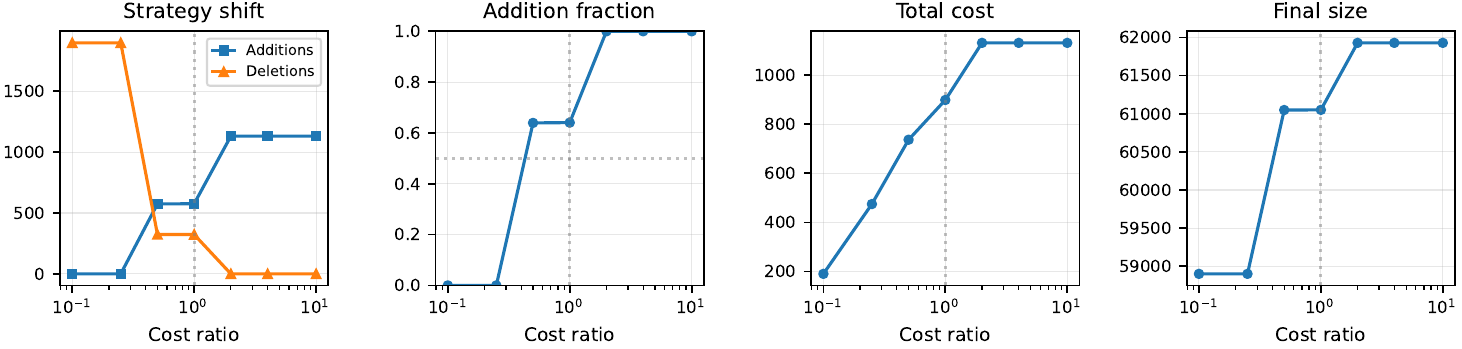}
    \caption{Effect of cost ratio on mitigation solutions for the COMPAS dataset.}
    \Description{}
    \label{fig:vary_cost_compas}
\end{figure}

\section{ML Evaluation}

This appendix provides supplementary material for Section~\ref{sec:ml_evaluation}, detailing the selection of the initial sample size $x_0$ and presenting the remaining experimental results.

\subsection{Selection of initial sample size}\label{app:x0_selection}
We now derive the maximum feasible initial sample fraction $x_0$ used in Section~\ref{sec:ml_evaluation}. Since the sampling is uniform with fraction $x$, each group-label pair $(\textbf{s}, y)$ contributes approximately $x \cdot [\textbf{s}y]$ instances to the initial sample and $(1-x) \cdot [\textbf{s}y]$ to the external pool. The mitigation computed on the initial sample requires additions that also scale with $x$: if the full dataset requires $\Delta^+_{\textbf{s},y}$ additions, the initial sample requires $x \cdot \Delta^+_{\textbf{s},y}$. For the pool to satisfy these additions, we need $(1-x) \cdot [\textbf{s}y] \geq x \cdot \Delta^+_{\textbf{s},y}$, which gives $x \leq \frac{[\textbf{s}y]}{[\textbf{s}y] + \Delta^+_{\textbf{s},y}}$. Taking the minimum over all group-label pairs with positive additions across all parameter configurations yields the maximum feasible fraction:
\begin{equation}
    x_{\max} = \min_{\substack{(\textbf{s}, y, p): \\ \Delta^+_{\textbf{s},y}(p) > 0}} \frac{[\textbf{s}y]}{[\textbf{s}y] + \Delta^+_{\textbf{s},y}(p)}
\end{equation}
where $\Delta^+_{\textbf{s},y}(p)$ denotes the additions for group-label $(\textbf{s}, y)$ under parameter configuration $p$. We set the initial sample size to $x_0 = \lfloor x_{\max} \cdot |T| \rfloor$.

\subsection{Additional experimental results}\label{app:ml_eval_experiments}

\paragraph{Fairness tolerance} 
The results are shown in Figures~\ref{fig:ml_eval_vary_epsilon_adult} (Adult) and ~\ref{fig:ml_eval_vary_epsilon_default} (Default).

\begin{figure}[htbp]
    \centering
    \begin{subfigure}[t]{0.35\textwidth}
        \centering
        \includegraphics[width=\textwidth]{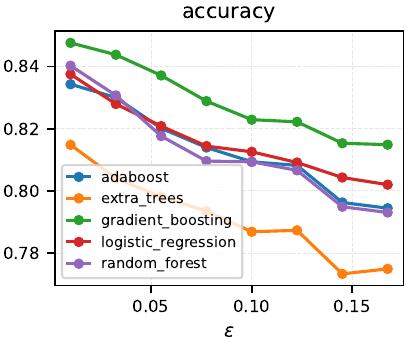}
        \caption{\texttt{min\_size} objective}
    \end{subfigure}
    \hfill
    \begin{subfigure}[t]{0.35\textwidth}
        \centering
        \includegraphics[width=\textwidth]{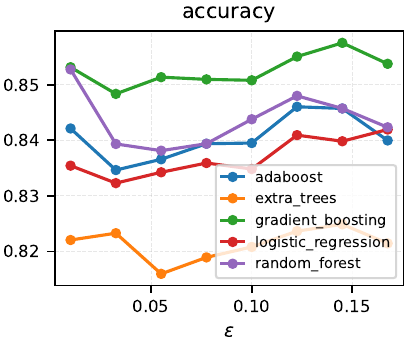}
        \caption{\texttt{min\_changes} objective}
    \end{subfigure}
    \caption{Fairness tolerance vs. Accuracy: Adult dataset.}
    \Description{}
    \label{fig:ml_eval_vary_epsilon_adult}
\end{figure}

\begin{figure}[htbp]
    \centering
    \begin{subfigure}[t]{0.35\textwidth}
        \centering
        \includegraphics[width=\textwidth]{figures/ml_eval/vary_epsilon/default_credit_sex_label_EDUCATION_default_payment_ml_eval_vary_epsilon_min_size_accuracy_compact.pdf}
        \caption{\texttt{min\_size} objective}
    \end{subfigure}
    \hfill
    \begin{subfigure}[t]{0.35\textwidth}
        \centering
        \includegraphics[width=\textwidth]{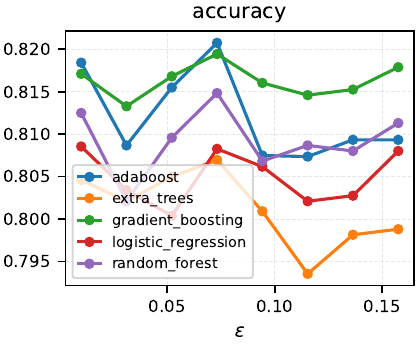}
        \caption{\texttt{min\_changes} objective}
    \end{subfigure}
    \caption{Fairness tolerance vs. Accuracy: Default dataset.}
    \Description{}
    \label{fig:ml_eval_vary_epsilon_default}
\end{figure}

\paragraph{Coverage requirements}
The results are shown in Figure~\ref{fig:ml_eval_vary_coverage_adult} (Adult), Figure~\ref{fig:ml_eval_vary_coverage_compas_min_size} (COMPAS \texttt{min\_size}) and Figure~\ref{fig:ml_eval_vary_coverage_default_min_changes} (Default \texttt{min\_changes}).

\begin{figure}[htbp]
    \centering
    \begin{subfigure}[t]{0.35\textwidth}
        \centering
        \includegraphics[width=\textwidth]{figures/ml_eval/vary_coverage/adult_sex_race_binary_income_ml_eval_vary_coverage_min_size_accuracy_compact.pdf}
        \caption{\texttt{min\_size} objective}
    \end{subfigure}
    \hfill
    \begin{subfigure}[t]{0.35\textwidth}
        \centering
        \includegraphics[width=\textwidth]{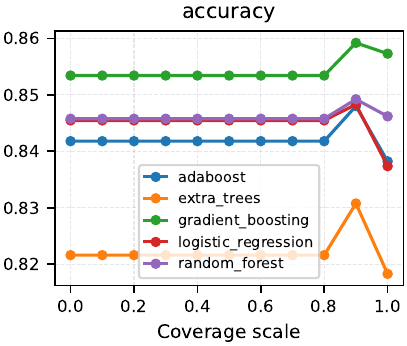}
        \caption{\texttt{min\_changes} objective}
    \end{subfigure}
    \caption{Coverage vs. Accuracy: Adult dataset.}
    \Description{}
    \label{fig:ml_eval_vary_coverage_adult}
\end{figure}

\begin{figure}
    \centering
    \includegraphics[scale=.8]%
    {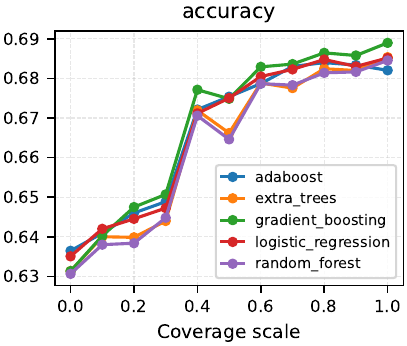}
    \caption{Coverage vs. Accuracy: COMPAS \texttt{min\_size} objective.}
    \Description{}
    \label{fig:ml_eval_vary_coverage_compas_min_size}
\end{figure}

\begin{figure}
    \centering
    \includegraphics[scale=.8]%
    {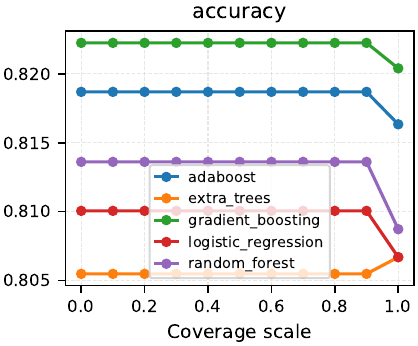}
    \caption{Coverage vs. Accuracy: Default \texttt{min\_changes} objective.}
    \Description{}
    \label{fig:ml_eval_vary_coverage_default_min_changes}
\end{figure}

\end{document}